\documentclass[lettersize,journal]{IEEEtran}
\usepackage{amsmath,amsfonts}
\usepackage{algorithmic}
\usepackage{algorithm}
\usepackage{array}
\usepackage{textcomp}
\usepackage{stfloats}
\usepackage{url}
\usepackage{verbatim}
\usepackage{graphicx}
\usepackage{cite}
\usepackage{tikz}
\usetikzlibrary{backgrounds}
\usepackage{pgfplots}
\usetikzlibrary{fit} %????
\usetikzlibrary{arrows.meta}% ????
\usepackage{xcolor}
\usepackage{pgfplots}
\usepackage{makecell}
\usepackage{amsmath}
\usetikzlibrary{intersections}
\usetikzlibrary{calc}
\usetikzlibrary{patterns}
\usetikzlibrary{positioning}
\usetikzlibrary{scopes}
\usepackage{amsfonts}
\usepackage{caption}
\usepackage{enumitem}
\usepackage{algorithm}
\usepackage{algorithmic}
\usepackage{float}
\usepackage{gensymb}

\usepackage{amsmath}
\usepackage{amsthm}
\newtheorem{prop}{Proposition}
\newtheorem{remark}{Remark}
\newtheorem{theorem}{Theorem}
\newtheorem{corollary}{Corollary}

\usepackage{booktabs}

\usepackage{graphicx}
\usepackage{subcaption}
\usepackage{cases}
\usepackage{empheq}

\usepackage{setspace}
\usepackage{balance}
\usepackage{soul}

\makeatletter
\newcommand*{\rom}[1]{\expandafter\@slowromancap\romannumeral #1@}
\makeatother

\tikzset{%
  do path picture/.style={%
    path picture={%
      \pgfpointdiff{\pgfpointanchor{path picture bounding box}{south west}}%
        {\pgfpointanchor{path picture bounding box}{north east}}%
      \pgfgetlastxy\x\y%
      \tikzset{x=\x/2,y=\y/2}%
      #1
    }
  },
  sin wave/.style={do path picture={    
    \draw [line cap=round] (-3/4,0)
      sin (-3/8,1/2) cos (0,0) sin (3/8,-1/2) cos (3/4,0);
  }},
  cross/.style={do path picture={    
    \draw [line cap=round] (-1,-1) -- (1,1) (-1,1) -- (1,-1);
  }},
  plus/.style={do path picture={    
    \draw [line cap=round] (-3/4,0) -- (3/4,0) (0,-3/4) -- (0,3/4);
  }}
}

\makeatletter
\let\oldtagform@\tagform@
\renewcommand{\eqref}[1]{\textup{\oldtagform@{\ref{#1}}}}
\makeatother

\definecolor{Blue}{RGB}{0,0,255}
\definecolor{Red}{RGB}{255,0,0}

\definecolor{Green}{RGB}{0, 128, 0}
\definecolor{Orange}{RGB}{255, 165, 0}

\definecolor{Blue_1}{RGB}{77,33,189}
\definecolor{Orange_1}{RGB}{247,144,61}
\definecolor{Green_1}{RGB}{89,169,90}
\usepackage{wrapfig, blindtext}

\usepackage[english]{babel}

\usepackage[T1]{fontenc}
\usepackage{multirow}

\newcommand*\xbar[1]{%
  \hbox{%
    \vbox{%
      \hrule height 0.5pt % The actual bar
      \kern0.5ex%         % Distance between bar and symbol
      \hbox{%
        \kern-0.1em%      % Shortening on the left side
        \ensuremath{#1}%
        \kern-0.1em%      % Shortening on the right side
      }%
    }%
  }%
} 

\usepackage{graphicx}

\begin{document}

\title{Towards Better Long-range Time Series Forecasting \\ using Generative Forecasting}

\author{
  Shiyu Liu, Rohan Ghosh, Mehul Motani \\
  Department of Electrical and Computing Engineering\\
  National University of Singapore\\
  shiyu\_liu@u.nus.edu, rghosh92@gmail.com, motani@nus.edu.sg}
        % <-this % stops a space
%\thanks{This paper was produced by the IEEE Publication Technology Group. They are in Piscataway, NJ.}% <-this % stops a space
%\thanks{Manuscript received April 19, 2021; revised August 16, 2021.}}

\markboth{Journal of \LaTeX\ Class Files,~Vol.~14, No.~8, August~2021}%
{Shell \MakeLowercase{\textit{et al.}}: A Sample Article Using IEEEtran.cls for IEEE Journals}


\maketitle

\begin{abstract}
Long-range time series forecasting is usually based on one of two 
existing forecasting strategies: Direct Forecasting and Iterative Forecasting, where the former provides low bias, high variance forecasts and the latter leads to low variance, high bias forecasts. In this paper, we propose a new forecasting strategy called Generative Forecasting (GenF), which generates synthetic data for the next few time steps and then makes long-range forecasts based on generated and observed data. We theoretically prove that GenF is able to better balance the forecasting variance and bias, leading to a much smaller forecasting error. We implement GenF via three components: (i) a novel conditional Wasserstein Generative Adversarial Network (GAN) based generator for synthetic time series data generation, called CWGAN-TS. (ii) a transformer based predictor, which makes long-range predictions using both generated and observed data. (iii) an information theoretic clustering algorithm to improve the training of both the CWGAN-TS and the transformer based predictor. The experimental results on five public datasets demonstrate that GenF significantly outperforms a diverse range of state-of-the-art benchmarks and classical approaches. Specifically, we find a 5\% - 11\% improvement in predictive performance (mean absolute error) while having a 15\% - 50\% reduction in parameters compared to the benchmarks. Lastly, we conduct an ablation study to further explore and demonstrate the effectiveness of the components comprising GenF. 
\end{abstract}

\begin{IEEEkeywords}
Time Series Forecasting, Forecasting Bias and Variance, Generative Adversarial Network and Ablation Study.
%(1) cite latest paper/tnnls (2) the gan paper (3) long-forecasting results
\end{IEEEkeywords}

\section{Introduction}
Accurate forecasting of time series data is an important problem in many sectors, such as energy and healthcare \cite{torres2021deep,NEURIPS2021_312f1ba2,bellot2021neural,bellot2021policy,liu2021gated, passalis2019deep}. In terms of prediction horizon, long-range forecasting (also called multi-step ahead forecasting) is often preferred than short-range forecasting (i.e., few time steps ahead) as it allows more time for early intervention and planning opportunities  \cite{li2019enhancing,rangapuram2018deep,cheng2020towards,NEURIPS2019_466accba,bandara2020lstm}. As an example, long-range forecasting of patient's vital signs effectively gives clinicians more time to take actions and may reduce the occurrence of potential adverse events \cite{Edward2017,jarrett2021clairvoyance,de2021hybrid}.

To perform long-range forecasting, there are two forecasting strategies: Direct Forecasting (DF) and Iterative Forecasting (IF). As the name suggests, DF directly makes predictions $N$ time steps ahead, but the forecasting performance tends to decrease (i.e., forecasting variance increases) as $N$ grows \cite{Mar2006}. In IF, the previous predictions are used as part of the input to recursively make predictions for the next time step. However, the predictions made in such a recursive and supervised manner is susceptible to error propagation, resulting in degraded forecasting performance (i.e., forecasting bias increases) as $N$ grows \cite{taieb2012review,taieb2015bias,sezer2020financial,lim2021time,informer}. 

In this paper, we improve the performance of long-range time series forecasting by proposing a new forecasting strategy. The contributions of our work are summarized as follows.

\begin{enumerate}[itemsep = 2mm,leftmargin=5mm, topsep=0.5pt]

\item We propose a new forecasting strategy called Generative Forecasting (GenF), which first generates synthetic data for the next few time steps and then makes long-range predictions based on generated and observed data. Theoretically, we prove that the proposed GenF is able to better balance the forecasting bias and variance, leading to a much smaller forecasting error.

\item We implement GenF via three components: (i) a new conditional Wasserstein Generative Adversarial Network (GAN) \cite{arjovsky2017wasserstein, mirza2014conditional, gulrajani2017improved,goodfellow2014generative} based generator for synthetic data generation called CWGAN-TS. (ii) a transformer based predictor, which makes long-range predictions using both generated and observed data. (iii) an information theoretic clustering (ITC) algorithm to improve the training of both the CWGAN-TS and the transformer based predictor.

\item We conduct experiments on five public time series datasets and the results demonstrate that GenF significantly outperforms a diverse range of state-of-the-art (SOTA)
benchmarks and classical approaches. Specifically, we find a 5\% - 11\% improvement in predictive performance (evaluated via mean absolute error) while having a 15\% - 50\% reduction in parameters compared to the SOTA benchmarks.

\item We conduct an ablation study to further evaluate and demonstrate the effectiveness of each component comprising GenF from the perspective of synthetic data generation and forecasting performance.
\end{enumerate}

We note that a short version of this work has been published in \cite{liu2019early}, and this work extends \cite{liu2019early} in four aspects: (i) We provide a theoretical justification for the proposed GenF and show that it is able to better balance the forecasting bias and variance, leading to a much smaller forecasting error. (ii) We improve \cite{liu2019early} by introducing the CWGAN-TS as synthetic data generator and a transformer based network as long-range predictor while \cite{liu2019early} used LSTM as both the synthetic data generator and long-range predictor. (iii) We evaluate the performance using five public datasets and compare GenF to more SOTA methods. (iv) An ablation study is conducted to investigate the effectiveness of each components in GenF.

%We evaluate the proposed MRwMR-BUR criterion from the perspective of achieving the goal of MIBFS stated in \eqref{NewP}. Our experimental results in Section \ref{sec3.5} demonstrate that the proposed MRwMR-BUR can better achieve the goal of MIBFS than MRwMR. This serves as a more solid motivation for the proposed MRwMR-BUR. (ii) In addition to the classification accuracy, the performance of MRwMR-BUR is also evaluated in terms of Feature Interpretability (i.e., if the selected feature subset is precise for practitioners to explore the hidden relationship between features and labels) and Classifier Generalization (i.e., if the selected feature subset generalizes well to various classifiers). (iii) We improve the competitiveness of MRwMR-BUR for performance oriented tasks and propose a classifier based approach to estimate UR, leading to another variant of MRwMR-BUR, called MRwMR-BUR-CLF. The MRwMR-BUR-CLF further improves the performance of MRwMR by 3.8\% - 5.5\% and it also outperforms three popular classifier based feature selection methods.

\section{Background}

{\bf Problem definition. } Suppose we have an observation window containing multivariate observations for past $M$ time steps \{$X_1$, $X_2$, $\cdots$,  $X_M$ | $X_{i} \in \mathbb{R}^K$\}, where $M$ is the observation window length, $K$ is the number of features per observation and $X_{i}$ is the observation at time step $i$ (see Fig.\ref{OP}). The task of time series forecasting is to find an approach to map past observations to the future value, i.e., \{$X_1$, $X_2$, $\cdots$,  $X_M$\} $\rightarrow$ $\xbar{X}_{M+N}$. We note that $N$ is the prediction horizon, indicating we plan to make predictions $N$ time steps ahead (i.e., at time step $M$+$N$ in Fig. \ref{OP}). Next, in Section \ref{RW}, we discuss related work and the application of GAN-based models. In Section \ref{BM}, we shortlist two classical models and five SOTA baselines for performance comparison.

%{\fontfamily{lmtt}\selectfont Prediction Horizon}
\begin{figure*}[!t]
\begin{center}
\resizebox{15.5cm}{!}{
\begin{tikzpicture}[font=\fontfamily{lmtt}]
	\draw[->] [blue!5!black, thick] (-0.3, 1.5) -- (13.3,1.5);	
	\node at (0,2) (c) {$X_{1}$};
	\node at (1,2) (c) {$X_{2}$};
	\node at (2.2,2) (c) {$\cdots$};
	\node at (3.2,2) (c) {$X_{M}$};
	\node at (4.6,2) (c) {$\xbar{X}_{M+1}$};
	\node at (5.9,2) (c) {$\xbar{X}_{M+2}$};
	\node at (7.2,2) (c) {$\cdots$};
	\node at (8.5,2) (c) {$\xbar{X}_{M+N-1}$};
	\node at (10.8,2) (c) {$\xbar{X}_{M+N}$};
	\draw (10.8,2) ellipse (0.8cm and 0.4cm);

    \draw [dashed,thick,->,Blue, line width=0.5mm](3.3,2.4) [out=45,in=135] to  (4.5,2.4);
    \draw [dashed,thick,->,Blue, line width=0.5mm](4.6,2.4) [out=60,in=135] to  (5.8,2.4);
    \draw [dashed,thick,->,Blue, line width=0.5mm](5.9,2.4) [out=60,in=135] to  (7.1,2.4);
    \draw [dashed,thick,->,Blue, line width=0.5mm](7.2,2.4) [out=60,in=135] to  (8.4,2.4);
    \draw [dashed,thick,->,Blue, line width=0.5mm](8.5,2.4) [out=40,in=140] to  (10.7,2.4);

    \draw [thick,->,Red, line width=0.4mm](3.2,2.4) [out=25,in=155] to  (10.8,2.4);

	\node [black] at (1.9,1.2) (c) {\footnotesize Observation Window ($M$)};
	\node [black] at (6.7,1.2) (c) {\footnotesize Synthetic Window ($N-1$)};
	\node [black] at (11.4,1.2) (c) {\footnotesize Prediction Horizon ($N$)};
	\node [black] at (12.6,1.8) (c) { \small Time Step};
	
    \draw[thick,black,->,Red, line width=0.5mm] (-0.2,3.0)--(0.8,3.0)node [right,black]{\footnotesize Direct Forecasting};
    \draw[thick,black,dashed,->,Blue, line width=0.5mm] (-0.2,3.4)--(0.8,3.4)node [right,black]{\footnotesize Iterative Forecasting};

	\draw [blue!5!black, thick](-0.3, 1.65) -- (3.7, 1.65) -- (3.7, 2.4) -- (-0.3, 2.4) -- (-0.3, 1.65);	
	\draw [blue!5!black, dashed, thick](4, 1.65) -- (9.5, 1.65) -- (9.5, 2.4) -- (4, 2.4) -- (4, 1.65);		
	\end{tikzpicture}}
	\vspace{-2mm}
	\caption{Illustration of Direct/Iterative Forecasting via Observation Window, Synthetic Window and Prediction Horizon. }
	\label{OP}
	\vspace{-4mm}
\end{center}
\end{figure*}
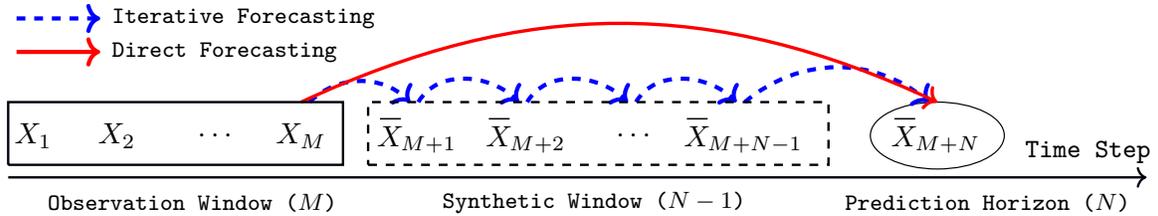

\subsection{Related Work}
\label{RW}

{\bf (1) Related work of time series forecasting.} The example of early methods using neural networks to perform long-range forecasting is \cite{nguyen2004multiple}, which proposed a group of networks to make predictions at different time steps. Along the way, several works attempt to improve the long-range forecasting by proposing new architectures. For example, \cite{yu2017long} proposed a Long Short-Term Memory (LSTM) \cite{LSTM} based Tensor-Train Recurrent Neural Network as a module for sequence-to-sequence framework \cite{sutskever2014sequence}, called TLSTM. \cite{DeepAR} proposed an autoregressive recurrent network called DeepAR to provide probabilistic forecasts. \cite{lai2018modeling} proposed a Long- and Short-term Time-series network (LSTNet) which incorporates with an attention-based layer and autoregressive models. More recently, transformers have shown superior performance in capturing long-range dependency than recurrent networks. Recent works aim to improve the transformer by simplifying its complexity. For example, LogSparse \cite{li2019enhancing} and Reformer \cite{Kitaev2020Reformer} use heuristics to reduce the complexity of self-attention mechanism from $\mathcal{O}(n^2)$ to $\mathcal{O}(n \log n)$, where $n$ is the sequence length. Moreover, Informer \cite{informer} aimed to reduce the complexity via ProbSparse and distilling operations. In terms of the forecasting strategy, all these methods can be classified into two main classes: direct and iterative forecasting.

In {\bf direct forecasting}, the model $f$ is trained to directly make predictions for the prediction horizon $N$, i.e., $\xbar{X}_{M+N} = f (X_1, \cdots,  X_M)$ (see Fig. \ref{OP}). The advantage is that the models trained for different values of $N$ are independent and hence, immune to error propagation. However, as $N$ grows, DF tends to provide low bias but high variance predictions \cite{taieb2012review}. This can be seen by considering an example where the best forecast is a linear trend. In this case, DF may yield a broken curve as it does not leverage the dependencies of synthetic data \cite{bon2012}. 

In {\bf iterative forecasting}, the model is trained to make predictions for the next time step only i.e., $\xbar{X}_{M+1} = f (X_1, \cdots,  X_M)$ (see Fig. \ref{OP}). The same model will be used over and over again and previous predictions are used together with the past observations to make predictions for the next time step (e.g., $\xbar{X}_{M+2} = f (X_2, \cdots,  X_M, \xbar{X}_{M+1})$). This process is recursively repeated to make predictions for next $N$ time steps. The previous predictions can be considered as synthetic data with a synthetic window length = $N$-1. However, the synthetic data generated in such a supervised and recursive manner is susceptible to error propagation, i.e., a small error in the current prediction becomes larger in subsequent predictions, leading to low variance but high bias predictions \cite{Sorja2007}. 
Based on IF, to address the issue of error propagation, RECTIFY \cite{taieb2012} rectifies the synthetic data to be unbiased and Seq2Seq based models \cite{sutskever2014sequence} extends the decoder by adding more sequential models and each sequential model is trained for a specific prediction horizon with different parameters. 

{\bf (2) Related work of GAN. }Recently, GAN based networks have demonstrated promising results in many generative tasks. The first GAN applied to time series data was C-RNN-GAN \cite{mogren2016c} which used LSTM as the generator and discriminator. %One follow-up work, RCGAN \cite{esteban2017real} made use of the conditioning information to generate synthetic medical data.
Along the way, many works have explored generating synthetic data to address various problems. As an example, \cite{frid2018gan} used synthetic data augmentation to improve the classification results, and \cite{yoon2019time} proposed TimeGAN which trains predictive models to perform one-step ahead forecasting. It is also worthy mentioning that \cite{wu2020adversarial} also leveraged GAN for time series forecasting. The key difference with our proposed GenF is that (i) The proposed GenF is a general framework and is flexible enough to support any model as synthetic data generator. (ii) The CWGAN-TS used in GenF is not directly involved during the training of forecasting model while the generator's output in \cite{wu2020adversarial} is the forecasting.

%However, to the best of our knowledge, the use of synthetic data generated by the GAN based networks to improve long-range forecasting remains largely unexplored.

\subsection{Selected Benchmark Methods}
\label{BM}
We shortlist five SOTA baselines discussed above: {\bf (i) TLSTM} (seq2seq based model), {\bf (ii) LSTNet} (attention based model), {\bf (iii) DeepAR} (autoregressive based model), {\bf (iv) LogSparse} (transformer based model) and {\bf (v) Informer} (transformer based model, direct forecasting) for comparison as they are reported to provide outstanding long-range forecasting performance \cite{lai2018modeling,informer}. Moreover, the authors of these methods have provided clear and concise source code, allowing us to correctly implement and tune these algorithms. In addition, two classical time series forecasting approaches: {\bf (i) LSTM and (ii) Autoregressive Integrated Moving Average (ARIMA)} \cite{ARIMA} are examined for comparison as well.

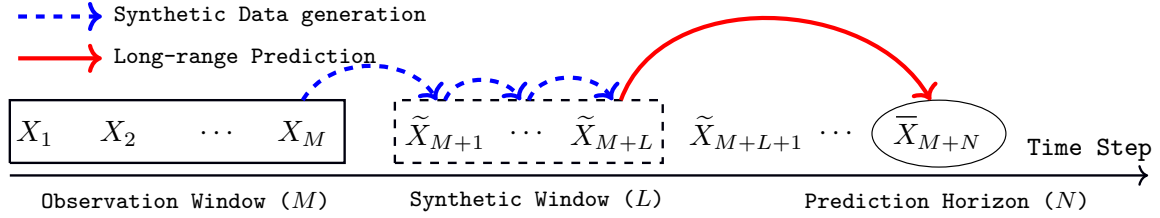
\begin{figure*}[!t]
\begin{center}
\resizebox{15.5cm}{!}{
\begin{tikzpicture}[font=\fontfamily{lmtt}]
	\draw[->] [blue!5!black, thick] (-0.3, 1.5) -- (13.3,1.5);	
	\node at (0,2) (c) {$X_{1}$};
	\node at (1,2) (c) {$X_{2}$};
	\node at (2.2,2) (c) {$\cdots$};
	\node at (3.2,2) (c) {$X_{M}$};
	\node at (4.9,2) (c) {$\widetilde{X}_{M+1}$};
	\node at (5.9,2) (c) {$\cdots$};
	\node at (6.9,2) (c) {$\widetilde{X}_{M+L}$};
	\node at (8.5,2) (c) {$\widetilde{X}_{M+L+1}$};
	\node at (9.6,2) (c) {$\cdots$};
	\node at (10.8,2) (c) {$\xbar{X}_{M+N}$};
	\draw (10.8,2) ellipse (0.8cm and 0.4cm);
	\node [black] at (1.8,1.2) (c) {\footnotesize Observation Window ($M$)};

    \node [black] at (6.0,1.2) (c) {\footnotesize Synthetic Window ($L$)};
	\node [black] at (10.9,1.2) (c) {\footnotesize Prediction Horizon ($N$)};
	\node [black] at (12.6,1.8) (c) { \small Time Step};
	\draw [blue!5!black, thick](-0.3, 1.65) -- (3.7, 1.65) -- (3.7, 2.4) -- (-0.3, 2.4) -- (-0.3, 1.65);	
	\draw [blue!5!black, dashed, thick](4.3, 1.65) -- (7.5, 1.65) -- (7.5, 2.4) -- (4.3, 2.4) -- (4.3, 1.65);

    \draw [thick,->,Blue,dashed, line width=0.5mm](3.2,2.4) [out=60,in=120] to  (4.8,2.4);
    \draw [thick,->,Blue,dashed, line width=0.5mm](4.9,2.4) [out=70,in=120] to  (5.8,2.4);
    \draw [thick,->,Blue,dashed, line width=0.5mm](5.9,2.4) [out=70,in=120] to  (6.9,2.4);
    
    \draw [thick,->,Red, line width=0.5mm](7.0,2.4) [out=70,in=120] to  (10.7,2.4);

    \draw[thick,black,->,Blue,dashed, line width=0.5mm] (-0.2,3.4)--(0.8,3.4)node [right,black]{\footnotesize Synthetic Data generation};
    \draw[thick,black,->,Red, line width=0.5mm] (-0.2,2.9)--(0.8,2.9)node [right,black]{\footnotesize Long-range Prediction};
    \end{tikzpicture}}
    %\vspace{-2mm}
    \caption{Illustration of the Proposed GenF via Observation Window, Synthetic Window and Prediction Horizon.}
     \label{GenF_Img}
    \vspace{-3mm}
\end{center}
\end{figure*}

\section{Generative Forecasting (GenF)}
\label{sec3}

We first introduce the idea of GenF in Section \ref{sec3.1}. Next, in Section \ref{sec3.2}, we theoretically prove that the proposed GenF is able to better balance the forecasting variance and bias, leading to a smaller forecasting error. Lastly, in Section \ref{sec3.3}, we detail the implementation of GenF.

\subsection{Idea of GenF}
\label{sec3.1}

To improve long-range time series forecasting, we develop an approach called Generative Forecasting (GenF), which consists of two steps (see our illustration in Fig. \ref{GenF_Img}): 
\begin{enumerate}[itemsep = 2mm,leftmargin=5mm, topsep=0.5pt]
    \item {\bf Synthetic Data Generation}: GenF first recursively generates synthetic synthetic data for next $L$ time steps (i.e., $\widetilde{X}_{M+1},..., \widetilde{X}_{M+L}$ in the dashed box in Fig. \ref{GenF_Img}) conditioned on the data of past $M$ time steps.
    \item {\bf Long-range Prediction}: GenF concatenates the past observations ($X_{1},..., X_{M}$) with the generated synthetic data ($\widetilde{X}_{M+1},...,\widetilde{X}_{M+L}$) and keeps a window size of $M$ by dropping the oldest observations, resulting in a sequence of ($X_{L+1},..., X_{M}$, $\widetilde{X}_{M+1},...,\widetilde{X}_{M+L}$). Finally, GenF makes long-range predictions for time step $M+N$ using ($X_{L+1},..., X_{M}$, $\widetilde{X}_{M+1},...,\widetilde{X}_{M+L}$) as input.
\end{enumerate}

The {\bf key difference} with direct forecasting and iterative forecasting is that GenF leverages synthetic data to shorten the effective prediction horizon and has a flexible synthetic window length of $L$ (see Fig. \ref{GenF_Img})), respectively. Unlike iterative forecasting of which the synthetic window length depends on the prediction horizon, the synthetic window length $L$ of GenF does not depends on the prediction horizon and is flexible. Adjusting the value of $L$ is a trade-off between forecasting variance and bias. A large value of $L$ brings GenF close to iterative forecasting, while a small value of $L$ brings GenF close to direct forecasting.

\subsection{Theoretical Results}
\label{sec3.2}

We now provide some theoretical insights into the behavior of forecasting error for the proposed GenF approach. In order to do so, we first undertake a bias-variance based approach for approximating the forecasting errors. Then, we estimate the variance of the low-bias direct forecasting step and the bias of the low-variance iterative forecasting step. Subsequently, we provide a theoretical result that bounds the forecasting error in terms of the bias and variance of the iterative and direct forecasting steps, respectively. Finally, we show that under certain conditions, the proposed GenF will yield much better performance. The proofs of the results given below are provided in the Appendix.

Let $Y$ = $\{{X_1, \cdots, X_M} \}$ be the past observations and $u_{M + N}$ = $\mathbb{E}[X_{M+N}|Y]$ be the conditional expectation of $N$-step ahead observation. Let $f(Y, \theta, N)$ be the $N$-step ahead forecast using $Y$ as the input with parameter $\theta \in \Theta$, where $\Theta$ represents the set of all possible trained parameter configurations from the dataset instances from the underlying data distribution. The Mean Squared Error (MSE) of a given strategy at prediction horizon $N$ can be decomposed as follows \cite{taieb2015bias}.
\vspace{2mm}\begin{align}
\text{MSE}_{N} &= \underbrace{\mathbb{E}_{Y}[(X_{M+N} - u_{M + N})^2|Y]}_\text{\small Noise, Z(N)} \nonumber \\ & + \underbrace{\mathbb{E}_{Y}[(u_{M + N} - \mathbb{E}_{\Theta} [f(Y,\theta, N)])^2]}_\text{\small Bias, B(N)} \nonumber \\ & + \underbrace{\mathbb{E}_{Y, \Theta}[(f(Y,\theta, N) - \mathbb{E}_{\Theta} [f(Y,\theta, N)])^2]}_\text{\small Variance, V(N)}. \label{theorem_eq}
\end{align}

The first term $Z(N)$ in \eqref{theorem_eq} is irreducible noise, which does not depend on the forecasting strategy used. The second term $B(N)$ is the forecasting bias and the third term $V(N)$ is the forecasting variance. Both $B(N)$ and $V(N)$ depend on the employed forecasting strategy and tend to grow with the prediction horizon $N$. The ideal configuration is to have a low bias and a low variance. However, this is never achieved in practice as decreasing bias will increase the variance and vice versa. Hence, a good forecasting strategy is to better balance bias and variance, so as to obtain the smallest MSE.

\begin{prop}\label{prop:1}
Let $S$ be the sum of bias and variance terms, we have $S_{dir} = $ $B_{dir}(N)$ + $V_{dir}(N)$ for direct forecasting and $S_{iter}$ = $B_{iter}(N)$ + $V_{iter}(N)$ for iterative forecasting. For GenF, let $Y_{M-L} = \{X_{L+1}, \cdots, X_{M}\}$ and $Y_{L} = \{{X}_{M+1}, \cdots, {X}_{M+L}\}$ be the past observations and let $\widetilde{Y}_{L} = \{\widetilde{X}_{M+1}, \cdots, \widetilde{X}_{M+L}\}$ be the generated synthetic data for the next $L$ time steps. Let $\gamma(\theta, N-L) =f(\{Y_{M-L},\widetilde{Y}_{L}\},\theta,N-L)-f(\{Y_{M-L},{Y}_{L}\},\theta,N-L) $. We then have, 

%$\gamma(\theta, N-L) =f(\{X_{L-1}, \cdots, \widehat{X}_{M+1},\cdots,\widehat{X}_{M+L} \},\theta,N-L)-f(\{X_{L-1},\cdots,{X}_{M+1},\cdots,{X}_{M+L} \},\theta,N-L) $. We then have, 

%$\widehat{X}_{M+L}$ be the estimate of $X_{M+L}$, from the iterative forecasting function. We have $S_{dir} = $ $B_{dir}(N)$ + $V_{dir}(N)$ for direct forecasting and $S_{iter}$ = $B_{iter}(N)$ + $V_{iter}(N)$ for iterative forecasting. Let $\gamma(\theta, N-L) =f(X_L, \cdots, \widehat{X}_{M+L},\theta,N-L)-f(X_{M+L},\theta,N-L) $. We then have, 
\end{prop}

\vspace*{-2mm}\begin{align}
    & S_{GenF} = \underbrace{B_{iter} (L) + V_{iter} (L)}_\text{\small Iterative Forecasting} \nonumber \\ &+ \underbrace{B_{dir} (N - L) + V_{dir} (N - L) + \mathbb{E}_{\theta \sim \Theta}[\gamma(\theta, N-L)^2]}_\text{\small Direct Forecasting} \label{eq2_main}
\end{align}

The proposition yields the joint bias and variance of the proposed GenF. Using this breakdown, the following Theorem and Corollary provide error bounds using this framework. 

%, using the bias-variance decompositions of its iterative and direct forecasting subparts.

\begin{theorem}\label{thm:1}
%We consider the direct forecasting with parameters $\theta_D$, iterative forecasting with parameters $\theta_I$, and the proposed GenF in Proposition \eqref{prop:1}. Assume that the ground truth realization of the forecasting process can be modelled by some $\theta_D^*$ and $\theta_I^*$, and after training, the estimated parameters follow $\theta_D \sim \mathcal{N}(\theta_D^*,\sigma_D^2)$ and $\theta_I \sim \mathcal{N}(\theta_I^*,\sigma_I^2)$. Assume that the iterative forecasting function is 2nd-order $L_1,L_2$-Lipschitz continuous, and the direct forecasting function is first order Lipschitz continuous. Let us denote quadratic recurrence relations of the form $b_{\alpha}(k+1)= b_{\alpha}(k)\left(L_1+1+b_{\alpha}(k)L_2 \right),$ where $b_{\alpha}(1)=\alpha \sigma_I^2$, for any $\alpha\geq0$. Assume that iterative forecasting has zero variance and direct forecasting has zero bias. Then, for some constants $\beta_0,\beta_1,\beta_2\geq0$, which represent the Lipschitz constants of the direct forecasting function, we have $S_{dir} \leq  U_{dir}$, $S_{iter} \leq U_{iter}$ and $S_{GenF} \leq U_{GenF}$, where $U_{dir} =  N\beta_1 + \sigma_{D}^2\beta_2$, $U_{iter} = b_{\alpha}(N)^2$, and $U_{GenF} = b_{\alpha}(L)^2(1+\beta_0) + (N-L)\beta_1 + \sigma_{D}^2\beta_2$. The quantities $\alpha$ and $\beta_0,\beta_1,\beta_2$ depend on the iterative and direct forecasting functions respectively.

We consider the direct forecasting with parameters $\theta_D$, iterative forecasting with parameters $\theta_I$, and the proposed GenF in Proposition \eqref{prop:1}. Assume that the ground truth realization of the forecasting process can be modelled by some $\theta_D^*$ and $\theta_I^*$, and after training, the estimated parameters follow $\theta_D \sim \mathcal{N}(\theta_D^*,\sigma_D^2)$ and $\theta_I \sim \mathcal{N}(\theta_I^*,\sigma_I^2)$. Assume that the iterative forecasting function is 2nd-order $L_1,L_2$-Lipschitz continuous, and the direct forecasting function is first order Lipschitz continuous. Let us denote quadratic recurrence relations of the form $b_{\alpha}(k+1)= b_{\alpha}(k)\left(L_1+1+b_{\alpha}(k)L_2 \right),$ where $b_{\alpha}(1)=\alpha \sigma_I^2$, for any $\alpha\geq0$. Assume that iterative forecasting has zero variance and direct forecasting has zero bias. Then, for some constants $\beta_0,\beta_1,\beta_2\geq0$, which represent the Lipschitz constants of the direct forecasting function, we have $S_{dir} \leq  U_{dir}$, $S_{iter} \leq U_{iter}$ and $S_{GenF} \leq U_{GenF}$, where $U_{dir} =  (N-1)\beta_1 + \sigma_{D}^2\beta_2$, $U_{iter} = b_{\alpha}(N)^2$, and $U_{GenF} = b_{\alpha}(L)^2(\beta_0) + (N-L-1)\beta_1 + \sigma_{D}^2\beta_2$. The quantities $\alpha$ and $\beta_0,\beta_1,\beta_2$ depend on the iterative and direct forecasting functions respectively.

\end{theorem}

\begin{corollary}\label{corr:1}
%$U_{dir}$, $U_{iter}$ and $U_{GenF}$ are as defined in Theorem \ref{thm:1}.When  $b_{\alpha}(1)^2(1+\beta_0)<\beta_1<b_{\alpha}(N)^2-b_{\alpha}(N-1)^2(1+\beta_0)-\sigma_D^2\beta_2$, we have that $U_{GenF} < U_{iter}$ and $U_{GenF} < U_{dir}$, for some $0<L<N$. Furthermore, when $N\beta_1 + \sigma_{D}^2\beta_2 \approx  b_{\alpha}(N)^2$, we have $U_{GenF} < U_{iter}$ and $U_{GenF} < U_{dir}$, for \textit{any} choice of $0<L<N$.

$U_{dir}$, $U_{iter}$ and $U_{GenF}$ are as defined in Theorem \ref{thm:1}.
When  $\beta_0<\min\{\beta_1/b_{\alpha}(1)^2,(b_{\alpha}(N)^2-\sigma_D^2\beta_2)/b_{\alpha}(N-1)^2\}$, we have that $U_{GenF} < U_{iter}$ and $U_{GenF} < U_{dir}$, for some $0<L<N$. Furthermore, when $(N-1)\beta_1 + \sigma_{D}^2\beta_2 \approx  b_{\alpha}(N)^2$, we have $U_{GenF} < U_{iter}$ and $U_{GenF} < U_{dir}$, for \textit{any} choice of $0<L<N$.
\end{corollary}

\begin{remark}
Theorem \ref{thm:1} provides upper bounds on the sum of bias and variance terms, and Corollary \ref{corr:1} provides the conditions under which the upper bounds for GenF are provably lower. One of the possible scenarios these conditions are satisfied is when the iterative forecaster has low single-horizon bias (i.e., $b_{\alpha}(1)$), and a large prediction horizon $N$. Corollary \ref{corr:1} suggests that under certain conditions, GenF can provide potentially smaller MSE, which is verified in our experiments in Section \ref{PC}. Corollary \ref{corr:1} also gives conditions under which GenF yields a smaller upper bound on error for any $0<L<N$.
\end{remark} 

\begin{figure*}[!t]
\begin{center}
\begin{tikzpicture}[font=\fontfamily{lmtt}, scale = 0.95]
	\draw[->] [blue!5!black, thick] (-1, 2.6) -- (16.5, 2.6);	
	\node at (0,2) (c) {$X_{1}$};
	\node at (1,2) (c) {$X_{2}$};
	\node at (2.2,2) (c) {$\cdots$};
	\node at (3.2,2) (c) {$X_{M}$};
	\node[text = blue] at (5.9,2) (c) {$\widetilde{X}_{M+1}$};
	\node[text = blue] at (7.4,2) (c) {$\cdots$};
	\node[text = blue] at (8.8,2) (c) {$\widetilde{X}_{M+L}$};
	\node at (10.9,2) (c) {$\widetilde{X}_{M+L+1}$};
	\node at (12.2,2) (c) {$\cdots$};
	\node at (13.4,2) (c) {$\xbar{X}_{M+N}$};
	\draw (13.4,2) ellipse (0.8cm and 0.4cm);
	
	\node [black] at (1.9,2.9) (c) {\small Observation Window(M)};
	\node [black] at (7.8,2.9) (c) {\footnotesize Synthetic Window(L)};
	\node [black] at (13.5,2.9) (c) {\footnotesize Prediction Horizon(N)};
	\node [black] at (15.2,2.4) (c) { \small Time Step};
	\draw [blue!5!black, thick](-0.3, 1.65) -- (3.7, 1.65) -- (3.7, 2.4) -- (-0.3, 2.4) -- (-0.3, 1.65);	
	\draw [blue!5!black, dashed, thick](5.3, 1.65) -- (9.5, 1.65) -- (9.5, 2.4) -- (5.3, 2.4) -- (5.3, 1.65);

    \draw [thick,->,Blue,dashed, line width=0.5mm](5.9,1.6) [out=-70,in=-120] to  (7.2,1.6);
    \draw [thick,->,Blue,dashed, line width=0.5mm](7.3,1.6) [out=-70,in=-120] to  (8.6,1.6);
 \end{tikzpicture}
 \vspace{-12mm}
\end{center}
\end{figure*}
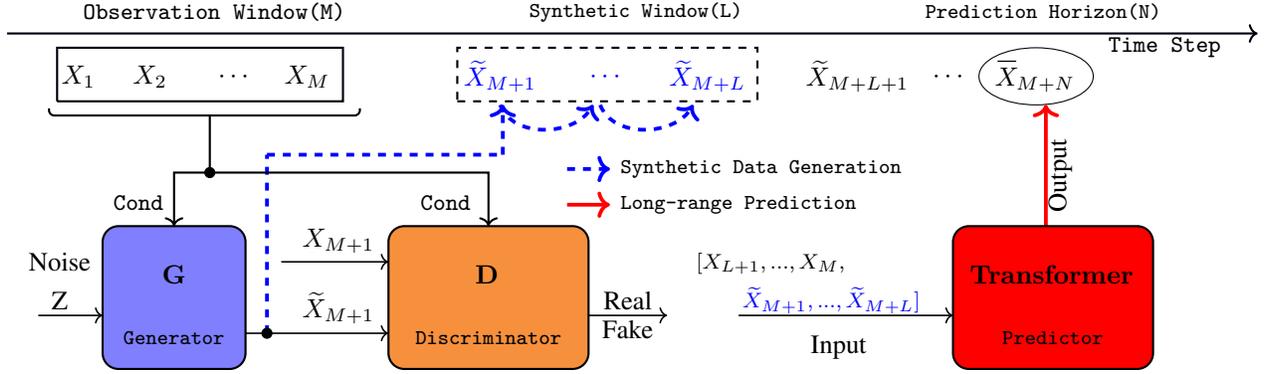
\begin{figure*}[!t]
\begin{center}
\hspace{-12mm}\begin{tikzpicture}[font=\fontfamily{lmtt},scale=0.95]
   \draw [thick,Blue,dashed, line width=0.5mm](-0.7,-0.5) to  (-0.7,2.0);
   \draw [thick,Blue,dashed, line width=0.5mm](-0.7, 2.0) to  (2.6,2.0); 
   \draw [thick,->, Blue,dashed, line width=0.5mm](2.6, 2.0) to  (2.6,2.7); 
   \draw [thick,->, Red, line width=0.5mm](10.2,0.5) to  (10.2,2.7);

   \draw[fill]  (-0.7,-0.5) circle (2pt);
   \draw[fill]  (-1.5, 1.75) circle (2pt);
   
    \draw[thick,black,->,Blue,dashed, line width=0.5mm] (3.5,1.8)--(4.1,1.8)node [right,black]{\footnotesize Synthetic Data Generation};
    \draw[thick,black,->,Red, line width=0.5mm] (3.5,1.3)--(4.1,1.3)node [right,black]{\footnotesize Long-range Prediction};
   
   \draw[line width=0.8pt,draw=black,fill=Blue!50,rounded corners=6pt](-3,-1) rectangle (-1,1) node [left=5cm,above=-0.9cm]{$\hspace{-19mm}\mathbf{G}$};
   \draw[line width=0.8pt,draw=black,fill=Orange_1,rounded corners=6pt](1,-1) rectangle (3.8,1) node [left=1.4cm,above=-0.9cm]{$\hspace{-27mm}\mathbf{D}$};
   
   \draw[line width=0.8pt,draw=black,fill=Red,rounded corners=6pt](8.9,-1) rectangle (11.7,1) node [left=1.4cm,above=-0.9cm]{$\footnotesize \hspace{-27mm} \mathbf{Transformer}$};

  \draw[line width=0.8pt, draw=black, <->](-2, 1)--(-2,1.75) -- (2.4, 1.75) -- (2.4, 1);
  
  \node[above=0.2cm]at (-2,-1){\footnotesize Generator};
  \node[above=0.2cm]at (2.4,-1){\footnotesize Discriminator};  
  \node[above=0.2cm]at (10.3,-1){\footnotesize Predictor};  
  \node[above=0.2cm]at (7.3,-1.2){Input};  
  \node[above=0.2cm, rotate = 90]at (10.7,1.5){Output};  
  %\node[above=0.2cm]at (-2.2,-1.9){CWGAN-TS Loss};  
   
  \draw[line width=0.8pt,draw=black,rounded corners=2pt](-1.5,1.75)--(-1.5,2.55)node[right = 1.5cm, above=-1cm]{};
  
  \node[above=0cm]at (-2.5,1.1){{\small Cond}};
  \node[above=0cm]at (1.8,1.1){{\small Cond}};

  \draw[line width=0.8pt,draw=black,fill=white,rounded corners=2pt](-3.75,2.65)--(-3.75,2.55)--(0.6,2.55)--(0.6,2.65);
      
   \node[above]at (-3.6,0.25){Noise};
   \node[above]at (-3.6,-0.3){Z};
   \node[above]at (0.3, 0.5){$X_{M+1}$};
   \node[above]at (0.3,-0.5){$\widetilde{X}_{M+1}$};
   \node[above]at (4.35,-0.3){Real};
     
  \draw[line width=0.6pt,draw=black,fill=white,->](-3.9,-0.25)--(-3,-0.25);
  \draw[line width=0.6pt,draw=black,fill=white,->](-1,-0.5)--(1,-0.5);
  \draw[line width=0.6pt,draw=black,fill=white,->](-0.5,0.5)--(1,0.5);
  \draw[line width=0.6pt,draw=black,fill=white,->](3.8,-0.25)--(4.9,-0.25);
  
  \node[above]at (4.35,-0.7){Fake};
  
  \node[above]at (6.35, 0.2){\footnotesize $[X_{L+1},..., X_{M},$};
  \node[above, text = blue]at (7.2, -0.35){\footnotesize $\widetilde{X}_{M+1},..., \widetilde{X}_{M+L}]$};

  \draw[line width=0.6pt,draw=black,fill=white,->](5.9,-0.25)--(8.9,-0.25);
  
  \end{tikzpicture}
  \vspace{2mm}
  \caption{Implementation of GenF (Left: CWGAN-TS; Right: Transformer based Predictor).}
  \label{GenF_Imp}
  \vspace{-4mm}
\end{center}
\end{figure*}

\subsection{Implementation of GenF}
\label{sec3.3}

We now detail the implementation of GenF via three components described as follows.

\noindent {\bf 1) CWGAN-TS: Synthetic Data Generation. } For the iterative forecasting part (i.e., $B_{iter}(L) + V_{iter}(L)$) of GenF in \eqref{eq2}, the bias term (i.e., $B_{iter}(L)$) tends to be more dominant due to the nature of iterative forecasting. Therefore, we should select a low bias model for the iterative forecasting part in \eqref{eq2_main}, so as to obtain a smaller MSE. A recent work \cite{zhao2018bias} evaluates the generative bias of a diver range of models and their results suggest that GAN based models tend to have a relatively lower generative bias then other models studied (e.g., VAE). Similar results are reported in \cite{hu2019exploring} as well. This motivates the use of a GAN based model in GenF for synthetic data generation.

We propose a Conditional Wasserstein GAN for synthetic time series data generation \cite{arjovsky2017wasserstein,mirza2014conditional,gulrajani2017improved}, called CWGAN-TS. The the unsupervised loss of CWGAN-TS is as follows.
\begin{align} \label{loss_1}
\mathcal{L}_{U}  = & \underset{\bar{X}_{M+1}\sim P_g}{\mathbb{E}}\!\!\!\!\!\!\!\![D(\xbar{X}_{M+1}|Y)] - \underset{X_{M+1} \sim P_r}{\mathbb{E}}\!\!\!\!\!\!\!\![D(X_{M+1}|Y)] \nonumber \\
& + \underset{\widehat{X}_{M+1} \sim P_{\widehat{X}}}{\lambda\mathbb{E}}\!\!\!\!\!\!\!\![(||\nabla_{\widehat{X}} D(\widehat{X}_{M+1}|Y)||_{2} - 1)^2],
\vspace*{3mm}\end{align}
where $Y = \{X_1, \cdots, X_M\}$ is the condition, $X_{M+1} \sim P_r$ is the real data at time step $M+1$, $\xbar{X}_{M+1} \sim P_g$ is the generated data for time step $M+1$ and $\widehat{X}_{M+1} \sim P_{\widehat{X}}$ is sampled from $\xbar{X}_{M+1}$ and $X_{M+1}$ with $\epsilon$ uniformly sampled between 0 and 1, i.e., $\widehat{X}_{M+1} = \epsilon \xbar{X}_{M+1} + (1 - \epsilon) X_{M+1}$ with 0 $\leq \epsilon \leq $1. The first two terms in \eqref{loss_1} correspond to Wasserstein distance which improves learning performance over the Jensen-Shannon divergence used in the original GAN. The third term is the gradient penalty, corresponding to the 1-Lipschitz constraint \cite{gulrajani2017improved}. To control the generation of synthetic data, we apply two terms to CWGAN-TS during training: (i) the past observation $Y$ as the condition. (ii) an additional supervised loss as follows.
\vspace*{1mm}\begin{align} \label{loss_2}
\mathcal{L}_{S} = \mathcal{L}_{U} + \eta ||X_{M+1} - \xbar{X}_{M+1}||_{2},
\vspace*{1mm}\end{align}
where $\eta$ is a hyperparameter that balances the two losses. Importantly, this is, in addition to the unsupervised minmax game played over classification accuracy, the generator additionally minimizes the supervised loss. By combining the objectives in this manner, CWGAN-TS aims to generate accurate synthetic data while preserving the temporal dynamics between conditioning input and generated data. The loss function of CWGAN-TS is further evaluated via an ablation study.

\begin{table*}[!h]
\hspace*{2mm}\begin{minipage}{0.5\textwidth}
{\renewcommand{\arraystretch}{1.3}
\setlength{\tabcolsep}{9pt}
\centering
\begin{tabular}{|c|c|c|c|}
\hline
\multicolumn{3}{|c|}{\small \bf Generative Forecasting (GenF)}\\
\hline
\multicolumn{2}{|c|}{\small CWGAN-TS} & {\small Predictor} \\
\hline
{\small Generator} & {\small Discriminator} & { - } \\
\hline
{\small $(S, M+1, K)$ } & {\small $ (S, M+1, K) $ } & {\small $(S, M, K)$} \\
\hline
{\small LSTM (5)} & {\small LSTM (5)} & {\small Atten En (3)}\\ \hline
{\small Linear (12)} & {\small Linear (12)} & {\small Atten En (3) } \\ \hline
{\small Linear (K)} & {\small Linear (4)} & {\small Atten De (3)} \\ \hline
{\small Reshape ($S$, 1, K)} & {\small Linear (1)} & {\small Atten De (3)} \\ \hline
{\small - } & {-} & {\small Linear (1)} \\ \hline
\end{tabular}}
\end{minipage}%
\hspace*{8mm}\begin{minipage}{0.5\textwidth}
{\renewcommand{\arraystretch}{1.07}
\footnotesize
\setlength{\tabcolsep}{10pt}
\centering
\begin{tabular}{|c|}
\hline
{\bf LSTNet} \\ skip-length p = 5, gradient clipping = 10, \\ epoch = 1000, dropout = 0.1, batchsize  = 64. \\ \hline
{\bf TLSTM} \\ learning rate (lr) decay = 0.8, lr = 1$\texttt{e}$-3, \\dropout = 0.1, batchsize = 64, epoch = 1000. \\  \hline
{\bf DeepAR} \\ LSTM (20), embedding size = 10 \\ batchsize = 64, LSTM dropout = 0.1, lr = 1$\texttt{e}$-3. \\ \hline
{\bf LogSparse} \\ embedding size = 16, kernel size = 9 \\ batchsize = 64, epoch = 1000, lr = 3$\texttt{e}$-3. \\ \hline
{\bf Informer} \\ Encoder/Decoder Layers = 2, heads = 4 \\ batchsize = 32, dropout = 0.05, lr = 1$\texttt{e}$-3. \\ \hline
\end{tabular}}
\end{minipage}
\vspace{2mm}
\caption{(left) Network configuration of GenF. The fourth row represents the shape of input variables and the parameter in Linear($\cdot$) is the number of output units. The parameter in LSTM($\cdot$) is the hidden state size. The parameter in Atten En and Atten De is the number of heads in multi-head encoder and multi-head decoder. (right) The training details of SOTA methods.}
\label{Fig:MI}
\end{table*}

\noindent {\bf 2) Transformer: Long-range Prediction.} Several recent works \cite{li2019enhancing,informer,wu2020adversarial,wu2020deep,farsani2021transformer,lim2021temporal,tang2021probabilistic} have shown the great potential of transformer in time series forecasting. We briefly introduce its architecture here and refer interested readers to \cite{vaswani2017attention} for more details.

In the self-attention layer, a multi-head self-attention sublayer takes input {\bf Y} (i.e., time series data and its positioning vectors) and computes the attention function simultaneously using query matrices: {\bf Q$_h$ = YW$^Q_h$}, key matrices {\bf K$_h$ = YW$^K_h$}, and value matrices {\bf V$_h$ = YW$^V_h$}. Note that {\bf W$^Q_h$}, {\bf W$^K_h$} and {\bf W$^V_h$} are learnable parameters. The scaled dot-product attention computes a sequence of vector outputs:
\begin{align}
{\bf O}_h  &= {\mathrm{Attention}}({\bf Q}_h, {\bf K}_h, {\bf V}_h) \nonumber \\
& = {\mathrm{softmax}}(\frac{{\bf Q}_h {\bf K}_h^T}{\sqrt{d_k}} \cdot {\bf M}) {\bf V}_h
\end{align}
where $\frac{1}{\sqrt{d_k}}$ is a scaled factor and {\bf M} is a mask matrix. In GenF, we concatenate {\bf O}$_h$ ($h = 1, \cdots, H$ and $H$ is number of heads) and pass the concatenation through several fully connected layers before arriving at the final output. The architecture details can be found in Table \ref{Fig:MI}.

\noindent {\bf 3) The ITC Algorithm.} The datasets studied in this paper contain time series data for different patients, countries and so on. In the following, we refer to them as units. Since GenF contains two steps, synthetic data generation and long-range forecasting, it requires two independent datasets: datasets $\mathbb{G}$ and $\mathbb{P}$ to train the CWGAN-TS and the transformer based predictor, respectively. It is possible to randomly split the entire training dataset at unit level into datasets $\mathbb{G}$ and $\mathbb{P}$, but the resulting datasets may not represent the entire training dataset well. We address this issue by suggesting an information theoretic clustering (ITC) algorithm based on Mutual Information (MI), denoted by $I (X;Y)$, which is a well-studied measure from information theory \cite{cover2006elements} that quantifies the dependency between random variables $X$ and $Y$. 

The ITC algorithm aims to select representative training subsets to better train the CWGAN-TS and the transformer based predictor. It consists of three steps:(i) Assign a score to each unit via the scoring function $J(P_i) = \sum_{P_j \in \mathbb{D}, P_j \neq P_i }^{} I(P_i, P_j)$, where $P_i$ refers to the candidate unit and $\mathbb{D} = \{P_1, P_2, \cdots  \}$ is the dataset containing all units. To estimate MI (i.e., $I(P_i, P_j)$), we use a nearest neighbor based approach called KSG estimator \cite{Kraskov2003} as all features studied are continuous variables. (ii) Divide all units into $\gamma$ groups based on the descending order of all scores, where $\gamma$ is a tunable parameter. The units with similar scores will be grouped together and units within the same group tend to be highly dependent on each other. (iii) Randomly sample from each subgroup. This is equivalent to selecting representative units of each subgroup. Random proportional sampling from all groups gives a new training dataset $\mathbb{G}$ and the remaining units form the new training dataset $\mathbb{P}$. In such a manner, we argue that we can select more representative units for better training. The effect of the ITC algorithm is further evaluated in an ablation study in Section \ref{SD}.

\section{Performance Evaluation}
\label{PE}
We summarize the dataset information in Section \ref{DataDes} and describe the experimental setup in Section \ref{ES}. Next, in Section \ref{PC}, we compare the performance of GenF to benchmark methods. Lastly, we conduct an ablation study to evaluate the effectiveness of our framework in Section \ref{SD}.

\subsection{Real-world Datasets}
\label{DataDes}
We shortlist five public time series datasets comprising different time series patterns (e.g., periodical, monotonic) and time intervals (e.g., hourly, daily, annual) from three popular domains (healthcare, environment and energy consumption).
We summarize each dataset as follows.
\begin{enumerate}[itemsep = 2mm,leftmargin=5mm, topsep=2pt]

\item The Vital Sign dataset from MIMIC-III \cite{MIMIC}, which contains 500 patients in the MIMIC-III Clinical database. For each patient, we extract 6 features: heart rate (bpm), respiratory rate, blood oxygen saturation (\%), body temperature ($\degree$F), systolic and diastolic blood pressure (mmHg). The vital signs are recorded at a hourly interval over a duration of 144 hours on average. 

\item The Multi-Site Air Quality dataset from UCI \cite{Dua:2019}, which includes air pollutants data from 12 sites. For each site, we extract the hourly record of PM10, $\text{SO}_{2}$, $\text{NO}_{2}$, $\text{O}_{3}$, PM2.5 and CO. We note that all features are in units of $\text{ug/m}^3$ and each site has 35,000 records on average. 

\item The World Energy Consumption dataset from World Bank \cite{worldenergy2019} , which contains data from 128 countries and each country contains three annual energy consumption indicators: electricity (kWh/capita), fossil fuel (\% of total) and renewable energy (\% of total) from 1971 to 2014.

\item The Greenhouse Gas dataset from UCI \cite{Dua:2019}, which records the greenhouse gas concentrations at 2921 locations. The data are spaced 6 hours apart over a period of 3 months.

\item The Household Electricity Consumption dataset \cite{energy}, which contains the cumulative daily electricity consumption (kWh) for 995 households over a duration of 3 years.

\end{enumerate}

For each dataset, a small amount (i.e., less than 5\%) of missing values are imputed using the last historical readings. Moreover, we scale all variables to [0,1] and reshape all scaled data via a sliding window, resulting in a dataset $\mathbb{D} \in \mathbb{R}^{S \times M \times K}$, where $S$ is the number of samples, $K$ is number of features and $M$ is the observation window length.

\begin{algorithm*}[!t]
\algsetup{linenosize=\normalsize}
\normalsize
    \caption{\normalsize The Algorithm of Generative Forecasting } 
    \label{algorithm2}
    \begin{algorithmic}[1]
    \REQUIRE {\fontfamily{bch}\selectfont (i)} Time series dataset $\mathbb{D} \in \mathbb{R}^{S \times M \times K}$; {\fontfamily{bch}\selectfont (ii)} $\lambda$, $\beta$=5, 1 (see (1)); {\fontfamily{bch}\selectfont (iii)} batch size = 64; {\fontfamily{bch}\selectfont (iv)} Adam (lr = 0.001).
    \STATE Randomly split $\mathbb{D}$ (unit level): training dataset $\mathbb{T} \in\mathbb{R}^{S_1 \times M \times K}$ (60\%), i.e., $\mathbb{T}:  \{ X_{1}^{i}, \cdots, X_{M}^{i} \}$ for  $i = 1, \cdots, S_{1}$ and $X \in \mathbb{R}^{K}$, test dataset $\mathbb{Q} \in \mathbb{R}^{S_2 \times M \times K}$ (20\%), validation dataset $\in \mathbb{R}^{S_3 \times M \times K}$ (20\%).
    \STATE Use ITC algorithm to split $\mathbb{T}$ (unit level): training dataset $\mathbb{H} \in\mathbb{R}^{S_4 \times M \times K}$ (50\%), training dataset $\mathbb{P} \in\mathbb{R}^{S_5 \times M \times K}$ (50\%).
    \STATE Training dataset $\mathbb{H}$ $\xrightarrow[\text{}]{\text{train}}$ the Generator (G) and Discriminator (D) in CWGAN-TS via loss function (1).
    \FOR {$j = 1$ to $L$}
    	\STATE (i) Synthetic Data Generation: G($\mathbb{P}$) = $\mathbb{Y}_1 \in \mathbb{R}^{S_5 \times 1 \times K}$, i.e., $\mathbb{Y}_{1}:  \{\xbar{X}_{M+j}^{i} \}$ for  $i = 1, \cdots, S_{5}$; G($\mathbb{Q}$) =  $\mathbb{Y}_2 \in \mathbb{R}^{S_2 \times 1 \times K}$.
	\STATE (ii) Prune the oldest observation: $\mathbb{P} \leftarrow \mathbb{P}\backslash X_j^i$, i.e., $\{ X_{j+1}^{i}, \cdots, X_{M+j-1}^i \}$ for  $i = 1, \cdots, S_{5}$; $\mathbb{Q} \leftarrow \mathbb{Q} \backslash X_j^i. $
	\STATE (iii) Concatenation:  $\mathbb{P} \leftarrow \mathbb{P} \oplus \mathbb{Y}_1,$ i.e., $\{ X_{j+1}^{i},\cdots, X_{M+j-1}^i, \xbar{X}_{M+j}^{i} \}$ for  $i = 1,\cdots, S_{5}$; $\mathbb{Q} \leftarrow \mathbb{Q} \oplus \mathbb{Y}_2.$ 	
    \ENDFOR
    \STATE Training Dataset $\mathbb{P}$ $\xrightarrow[\text{}]{\text{train}}$ transformer based Predictor, Testing Dataset $\mathbb{Q}$ $\xrightarrow[\text{}]{\text{test}}$ transformer based Predictor. \\
    \STATE {\bf Return} MSE, MAE and sMAPE between predicted values and real values.
  \end{algorithmic}
\end{algorithm*}

\subsection{Experiment Setup \& Parameter Tuning}
\label{ES}
In the experiment, the dataset $\mathbb{D}$ is randomly split into three subsets at unit level: training dataset $\mathbb{T} \in \mathbb{R}^{S_1 \times M \times K}$ (60\%), test dataset $\mathbb{Q} \in \mathbb{R}^{S_2 \times M \times K}$ (20\%) and validation dataset $\in \mathbb{R}^{S_3 \times M \times K}$ (20\%). We note that both GenF and the benchmark methods are trained using the training dataset $\mathbb{T}$, and the test dataset $\mathbb{Q}$ is used to evaluate the performance. {\bf To ensure fair comparison, we conduct grid search over all tunable hyper-parameters and possible configurations using the validation dataset}. We highlight that all methods share the same grid search range and step size. Specifically, some key parameters are tuned as follows. (i) the hidden state size of LSTM is tuned from 5 to 100 with step size of 5. (ii) the size of fully connected layers are tuned from 1 to 10 with step size of 1. (iii) the number of heads in mutli-head self-attention layer are tuned from 1 to 10 with step size of 1.

We now provide a list of tuned parameters for predicting systolic blood pressure using the MIMIC-III Vital Signs dataset. In terms of the classical models: (i) ARIMA (2,0,1) is trained using the past $M$ historical values of a single feature (e.g., systolic blood pressure). (ii) for the canonical transformer, it contains two encoder layers, two decoder layers and three heads in the multi-head self-attention layer. (iii) for LSTM, we stack two LSTM layers (each LSTM with hidden size of 10) and two fully connected layers (with size of 10 and 1, respectively). This LSTM based neural network is directly trained using training dataset $\mathbb{T}$ for 1000 epochs using Adam \cite{adam} with a learning rate of 0.001, in batches of 64. As for the five SOTA methods, we refer to the source code released by their authors and some key parameters are as follows: LSTNet has a skip-length of 5 and a gradient clipping of 10, TLSTM uses a dropout of 0.1 and LogSparse uses a kernel size of 9 and an embedding size of 16. In terms of GenF, we use the LSTM and fully connected layer to implement the CWGAN-TS. The transformer based predictor consists of two decoder layers and two encoder layers, with three heads. For the detailed algorithm of GenF and other training details on five SOTA methods, the details are provided in Algorithm \ref{algorithm2} and Table \ref{Fig:MI}. {\bf  Furthermore, we use Tesla V100 devices for our experiments, and the source code will be released for reproducibility at the camera-ready stage.}

In the experiments, we compare the performance of GenF to two classical models and five SOTA methods. The Mean Squared Error (MSE) and Mean Absolute Error (MAE) are used to evaluate the performance, where the former captures both the variance and bias of the predictor, and the latter is useful to understand whether the size of the error is of concern or not. Furthermore, a scale invariant error metric called symmetric Mean Absolute Percentage Error (sMAPE) is also used. In Table \ref{performance_short}, we show the performance of several variants of GenF (i.e., GenF-3), where the 'X' in GenF-X represents the value of the synthetic window length $L$.

\begin{table*}[!t]
\small
\centering
\setlength{\tabcolsep}{11pt}
{\renewcommand{\arraystretch}{1.0}
\begin{tabular}{ll|cc|cc|cc|ccc}
\toprule
\multicolumn{2}{c|}{{\fontfamily{lmtt}\selectfont Prediction Horizon}} & \multicolumn{2}{c|}{{\fontfamily{lmtt}\selectfont t + 8}} & \multicolumn{2}{c|}{{\fontfamily{lmtt}\selectfont t + 12}} & \multicolumn{2}{c|}{{\fontfamily{lmtt}\selectfont t + 30}} & \multicolumn{2}{c}{{\fontfamily{lmtt}\selectfont t + 60}} \\ \midrule
& {\fontfamily{lmtt}\selectfont Metrics} &{\fontfamily{lmtt}\selectfont MAE} & {\fontfamily{lmtt}\selectfont sMAPE} &{\fontfamily{lmtt}\selectfont MAE} & {\fontfamily{lmtt}\selectfont sMAPE} &{\fontfamily{lmtt}\selectfont MAE} & {\fontfamily{lmtt}\selectfont sMAPE} &{\fontfamily{lmtt}\selectfont MAE} & {\fontfamily{lmtt}\selectfont sMAPE}\\
%\midrule
\cmidrule(lr){1-2} \cmidrule(lr){3-4}\cmidrule(lr){5-6}\cmidrule(lr){7-8}\cmidrule(lr){9-10}

 & {\fontfamily{lmtt}\selectfont ARIMA}      & 8.3$\pm${\tiny 1.3} & 7.3$\pm${\tiny 1.1} & 9.8$\pm${\tiny 0.8} & 8.4$\pm${\tiny 0.7} & 16.2$\pm${\tiny 1.7} & 13.1$\pm${\tiny 1.4} & 18.9$\pm${\tiny 2.7} & 14.5$\pm${\tiny 2.4} \\  
 
 \parbox[t]{2mm}{\multirow{8}{*}{\rotatebox[origin=c]{90}{\vspace*{-5mm}{\bf SOTA}}}} & {\fontfamily{lmtt}\selectfont LSTM} & 7.3$\pm${\tiny 0.6} & 6.4$\pm${\tiny 0.3} & 8.9$\pm${\tiny 1.1} & 7.6$\pm${\tiny 0.9}  & 13.9$\pm${\tiny 1.3} & 11.0$\pm${\tiny 0.9}   & 17.8$\pm${\tiny 2.9} & 13.3$\pm${\tiny 1.9}    \\  \midrule

 & {\fontfamily{lmtt}\selectfont TLSTM}          & 6.8$\pm${\tiny 0.4}    & 5.7$\pm${\tiny 0.4} & 8.2$\pm${\tiny 0.6}   & 7.1$\pm${\tiny 0.8} & 12.3$\pm${\tiny 1.2}   & 10.9$\pm${\tiny 1.1} & 15.0$\pm${\tiny 2.3} & 12.7$\pm${\tiny 1.5}  \\ 
 & {\fontfamily{lmtt}\selectfont LSTNet}         & 6.9$\pm${\tiny 0.6}    & 5.8$\pm${\tiny 0.4} & 8.2$\pm${\tiny 0.7}  & 7.0$\pm${\tiny 0.7}& 12.0$\pm${\tiny 1.3}  & 10.6$\pm${\tiny 1.3} & 14.2$\pm${\tiny 2.6}  & 11.5$\pm${\tiny 0.9} \\ 
  & {\fontfamily{lmtt}\selectfont DeepAR}       & 6.8$\pm${\tiny 0.5} & 6.0$\pm${\tiny 0.6} & 8.4$\pm${\tiny 1.1} & 7.2$\pm${\tiny 0.9} & 12.9$\pm${\tiny 1.1} & 11.2$\pm${\tiny 0.9} & 16.3$\pm${\tiny 2.8} & 12.9$\pm${\tiny 1.7} \\ 
  & {\fontfamily{lmtt}\selectfont Informer}       & 6.5$\pm${\tiny 0.4} &  5.3$\pm${\tiny 0.5} & 7.9$\pm${\tiny 0.7} & 6.6$\pm${\tiny 0.7} & 11.8$\pm${\tiny 1.5} & 10.1$\pm${\tiny 1.4} & 14.1$\pm${\tiny 2.7} & 11.5$\pm${\tiny 1.6} \\ 
 & {\fontfamily{lmtt}\selectfont LogSparse}      & 6.6$\pm${\tiny 0.7}   & 5.5$\pm${\tiny 0.3} & 8.1$\pm${\tiny 0.5}   & 6.9$\pm${\tiny 0.6}  & 11.6$\pm${\tiny 1.0}  & 9.7$\pm${\tiny 0.9}  & 14.5$\pm${\tiny 2.8} & 11.3$\pm${\tiny 1.2}   \\ \midrule      
 & {\fontfamily{lmtt}\selectfont GenF-3 (Ours)}          & {\bf 6.2$\pm${\tiny 0.4}}  &  {\bf 5.1$\pm${\tiny 0.5}} & 7.5$\pm${\tiny 0.6}  & 6.3$\pm${\tiny 0.5}   & 11.2$\pm${\tiny 1.4} & 9.2$\pm${\tiny 1.0}  & 13.5$\pm${\tiny 2.7} & 10.8$\pm${\tiny 1.9}  \\ 
 & {\fontfamily{lmtt}\selectfont GenF-6 (Ours)}          & 6.3$\pm${\tiny 0.4}  &   5.2$\pm${\tiny 0.3} & {\bf 7.4$\pm${\tiny 0.7}}  & {\bf 6.2$\pm${\tiny 0.6}}   & {\bf 10.7$\pm${\tiny 1.1}} & {\bf 8.9$\pm${\tiny 0.9}}  & {\bf 12.6$\pm${\tiny 2.5}} &  {\bf 10.2$\pm${\tiny 1.4}}  \\ 
 
\bottomrule
\end{tabular}}
\caption{Performance (MAE, sMAPE (\%) $\pm$ standard deviation over 5 runs) of {\bf predicting blood pressure using the MIMIC-III Vital Sign dataset}. The X in GenF-X is the synthetic window length $L$.} 
\vspace{4mm}
\label{performance_short}
\end{table*}

\begin{table*}[!ht]
\small
\centering
\setlength{\tabcolsep}{10pt}
{\renewcommand{\arraystretch}{1.0}
\begin{tabular}{ll|cc|cc|cc|cc}
\toprule
&{\fontfamily{lmtt}\selectfont Prediction Horizon} & \multicolumn{2}{c|}{{\fontfamily{lmtt}\selectfont t + 8}} & \multicolumn{2}{c|}{{\fontfamily{lmtt}\selectfont t + 12}} & \multicolumn{2}{c|}{{\fontfamily{lmtt}\selectfont t + 30}} & \multicolumn{2}{c}{{\fontfamily{lmtt}\selectfont t + 60}}  \\ \midrule 
 &{\fontfamily{lmtt}\selectfont Metrics} & {\fontfamily{lmtt}\selectfont MAE} & {\fontfamily{lmtt}\selectfont sMAPE} &{\fontfamily{lmtt}\selectfont MAE} & {\fontfamily{lmtt}\selectfont sMAPE} &{\fontfamily{lmtt}\selectfont MAE} & {\fontfamily{lmtt}\selectfont sMAPE} &{\fontfamily{lmtt}\selectfont MAE} & {\fontfamily{lmtt}\selectfont sMAPE} \\ 
\cmidrule(lr){1-2} \cmidrule(lr){3-4}\cmidrule(lr){5-6}\cmidrule(lr){7-8}\cmidrule(lr){9-10}

 &{\fontfamily{lmtt}\selectfont ARIMA}      & 21.3$\pm${\tiny 1.7} & 26$\pm${\tiny 6} & 27.8$\pm${\tiny 1.5} & 29$\pm${\tiny 9} & 29.2$\pm${\tiny 2.6} & 34$\pm${\tiny 11} & 31.9$\pm${\tiny 1.8} & 35$\pm${\tiny 9} \\  
 
 &{\fontfamily{lmtt}\selectfont LSTM} & 19.7$\pm${\tiny 1.5} & 22$\pm${\tiny 3} & 24.6$\pm${\tiny 1.4} & 27$\pm${\tiny 8}  & 25.9$\pm${\tiny 1.7} & 31$\pm${\tiny 9}   & 29.4$\pm${\tiny 3.1} & 33$\pm${\tiny 6}    \\  \midrule

\parbox[t]{2mm}{\multirow{5}{*}{\rotatebox[origin=c]{90}{\vspace*{-5mm}{\bf SOTA}}}}& {\fontfamily{lmtt}\selectfont TLSTM}        & 18.8$\pm$ {\tiny 1.0}   & 20$\pm${\tiny 5} & 20.7$\pm${\tiny 1.9}  & 22$\pm${\tiny 8} & 24.0$\pm${\tiny 2.0}    & 28$\pm${\tiny 9} & 27.5$\pm${\tiny 2.3} & 31$\pm${\tiny 9} \\ 
&{\fontfamily{lmtt}\selectfont LSTNet}    & 17.8$\pm$ {\tiny 2.0}    & 18$\pm${\tiny 4} &19.9$\pm${\tiny 1.8}  & 21$\pm${\tiny 9} & 23.6$\pm${\tiny 2.2}       & 25$\pm${\tiny 6} & 27.0$\pm${\tiny 2.5} & 28$\pm${\tiny 10}   \\ 
&{\fontfamily{lmtt}\selectfont DeerAR}     & 19.2$\pm$ {\tiny 1.3}    & 21$\pm${\tiny 5} & 22.4$\pm${\tiny 2.0}  & 25$\pm${\tiny 7} & 24.3$\pm${\tiny 2.1}  & 29$\pm${\tiny 7} & 28.2$\pm${\tiny 2.9} & 30$\pm${\tiny 8}  \\ 
&{\fontfamily{lmtt}\selectfont Informer}     & 18.2$\pm$ {\tiny 1.4}    & 20$\pm${\tiny 3} & 20.4$\pm${\tiny 1.2}  & 22$\pm${\tiny 8} & 23.1$\pm${\tiny 1.9}  & 24$\pm${\tiny 5} & 26.1$\pm${\tiny 2.5} & 29$\pm${\tiny 7}  \\ 
&{\fontfamily{lmtt}\selectfont LogSparse}     & 18.1$\pm$ {\tiny 1.5}    & 20$\pm${\tiny 4} & 21.5$\pm${\tiny 1.5}  & 24$\pm${\tiny 6} & 23.5$\pm${\tiny 2.3}  & 26$\pm${\tiny 5} & 26.7$\pm${\tiny 2.6} & 29$\pm${\tiny 8}  \\ 
\midrule
&{\fontfamily{lmtt}\selectfont GenF-2 (Ours)}    & {\bf 16.3$\pm$ {\tiny 1.8}}  & 16$\pm${\tiny 4}  & {\bf 18.0$\pm${\tiny 1.7}}       & {\bf 19$\pm${\tiny 6}}                  & 20.5$\pm${\tiny 2.2} & 23$\pm${\tiny 9} & 24.2$\pm${\tiny 2.5} & 27$\pm${\tiny 12}  \\ 
&{\fontfamily{lmtt}\selectfont GenF-3 (Ours)}    & 16.5$\pm$ {\tiny 1.8}    & 16$\pm${\tiny 5}  &18.2$\pm${\tiny 2.0}  & 19$\pm${\tiny 4} & 20.3$\pm${\tiny 2.4}          & 23$\pm${\tiny 8} & 24.0$\pm${\tiny 2.4}  & 26$\pm${\tiny 9} \\ 
&{\fontfamily{lmtt}\selectfont GenF-6 (Ours)}    & 16.9$\pm$ {\tiny 1.3}   & 17$\pm${\tiny 3}  &18.5$\pm${\tiny 1.9}    &  19$\pm${\tiny 3} & {\bf 19.7$\pm${\tiny 2.0}} & {\bf 22$\pm${\tiny 5}} & {\bf 22.9$\pm${\tiny 1.8}} & {\bf 25$\pm${\tiny 8}} \\
\bottomrule
\end{tabular}}
\caption{Performance (MAE, sMAPE (\%) $\pm$ standard deviation over 5 runs) of {\bf predicting $\text{NO}_{2}$ emission using the Multi-Site Air Quality dataset}. The bold indicates the best performance.}
\label{performance_2}
\end{table*}

\subsection{Performance Comparison}
\label{PC}

{\bf (1) GenF VS SOTA Methods.} In Table \ref{performance_short}, we demonstrate that GenF greatly outperforms all methods studied for predicting blood pressure on the Vital Sign dataset. For example, the MAE of GenF-3 at $t + 8$ is 8.8\% better than the classical seq2seq based method (TLSTM), 10.1\% better than the attention based method (LSTNet), 5\% better than the transformer based method (Informer). Interestingly, as the prediction horizon grows, generating more synthetic data could be helpful, leading to better forecasting performance. Specifically, comparing to the best performing benchmark, GenF-6 achieves an improvement of 7.7\% at $t + 30$ and 10.6\% at $t + 60$. Similar performance trends can be observed using the other four datasets (see Tables \ref{performance_1} - \ref{performance_4}). As an example, in Table \ref{performance_2}, where we summarize the performance of predicting $\text{NO}_{2}$ emission using the Multi-Site Air Quality dataset, we observe that GenF-6 obtains a 12.2\% lower MAE than Informer.

{\bf (2) Complexity Comparison.} We note that GenF is essentially a transformer based method as it uses a shallow transformer based predictor. When predicting the blood pressure, the parameter count for GenF and the other two transformer based SOTA methods are: (i) GenF: 9.0K (CWGAN-TS: 3K; Transformer based Predictor: 6K), (ii) Informer: 10.6K and (iii) LogSparse: 17.9K. GenF uses 15\% and 50\% less parameters than Informer and LogSparse, respectively, but has better performance than both. These SOTA methods aim to better capture long-range dependencies with deep transformers. GenF achieves the same goal by extending the existing time series with synthetic data. As a result, our shallow transformer can achieve better performance than deep transformers. In fact, the proposed GenF is a general framework and is flexible enough to support any model as predictor. The effect of using deep transformers as predictor is discussed in Section \ref{discussion}.

\begin{table*}[!ht]
\small
\centering
\setlength{\tabcolsep}{10pt}
\begin{tabular}{ll|cc|cc|cc|cc}
\toprule
&{\fontfamily{lmtt}\selectfont Prediction Horizon} & \multicolumn{2}{c|}{{\fontfamily{lmtt}\selectfont t + 8}} & \multicolumn{2}{c|}{{\fontfamily{lmtt}\selectfont t + 12}} & \multicolumn{2}{c|}{{\fontfamily{lmtt}\selectfont t + 30}} & \multicolumn{2}{c}{{\fontfamily{lmtt}\selectfont t + 60}}  \\ \midrule 
 &{\fontfamily{lmtt}\selectfont Metrics} & {\fontfamily{lmtt}\selectfont MAE} & {\fontfamily{lmtt}\selectfont sMAPE} &{\fontfamily{lmtt}\selectfont MAE} & {\fontfamily{lmtt}\selectfont sMAPE} &{\fontfamily{lmtt}\selectfont MAE} & {\fontfamily{lmtt}\selectfont sMAPE} &{\fontfamily{lmtt}\selectfont MAE} & {\fontfamily{lmtt}\selectfont sMAPE} \\ 
\cmidrule(lr){1-2} \cmidrule(lr){3-4}\cmidrule(lr){5-6}\cmidrule(lr){7-8}\cmidrule(lr){9-10}

 &{\fontfamily{lmtt}\selectfont ARIMA}      & 6.3$\pm${\tiny 1.2} & 7.6$\pm${\tiny 0.9} & 7.1$\pm${\tiny 1.3} & 8.8$\pm${\tiny 1.4} & 12.4$\pm${\tiny 1.9} & 11.2$\pm${\tiny 1.8} & 13.5$\pm${\tiny 2.2} & 12.4$\pm${\tiny 2.3} \\  
 
 &{\fontfamily{lmtt}\selectfont LSTM} & 5.9$\pm${\tiny 0.9} & 7.2$\pm${\tiny 0.7} & 6.6$\pm${\tiny 1.1} & 8.1$\pm${\tiny 1.2}  & 11.3$\pm${\tiny 1.4} & 10.6$\pm${\tiny 1.4}   & 12.1$\pm${\tiny 1.7} & 11.6$\pm${\tiny 1.9}    \\  \midrule

\parbox[t]{2mm}{\multirow{5}{*}{\rotatebox[origin=c]{90}{\vspace*{-5mm}{\bf SOTA}}}}& {\fontfamily{lmtt}\selectfont TLSTM}        & 5.8$\pm${\tiny 0.8}   & 7.0$\pm${\tiny 0.6} & 6.5$\pm${\tiny 0.9}  & 7.9$\pm${\tiny 1.1} & 11.0$\pm${\tiny 1.2}    & 10.4$\pm${\tiny 1.2} & 12.4$\pm${\tiny 1.8} & 11.8$\pm${\tiny 2.0} \\ 
&{\fontfamily{lmtt}\selectfont LSTNet}    & 6.0$\pm${\tiny 0.9}    & 6.8$\pm${\tiny 0.5} &6.6$\pm${\tiny 1.0}  & 8.1$\pm${\tiny 1.2} & 11.5$\pm${\tiny 1.6}  & 10.8$\pm${\tiny 1.4} & 12.6$\pm${\tiny 1.5} & 11.7$\pm${\tiny 2.1}   \\
&{\fontfamily{lmtt}\selectfont DeerAR}     & 6.4$\pm${\tiny 1.0}    & 7.5$\pm${\tiny 0.7} & 6.3$\pm${\tiny 0.9}  & 7.9$\pm${\tiny 1.0} & 11.8$\pm${\tiny 1.7}  & 10.6$\pm${\tiny 1.5} & 11.8$\pm${\tiny 1.6} & 11.2$\pm${\tiny 1.7}  \\ 
&{\fontfamily{lmtt}\selectfont Informer}     & 5.5$\pm${\tiny 0.6}    & 6.8$\pm${\tiny 0.5} & 6.3$\pm${\tiny 1.0}  & 8.2$\pm${\tiny 1.1} & 10.7$\pm${\tiny 1.3}  & 10.1$\pm${\tiny 1.1} & 11.5$\pm${\tiny 1.5} & 10.9$\pm${\tiny 1.5}  \\ 

&{\fontfamily{lmtt}\selectfont LogSparse}     & 5.7$\pm${\tiny 0.8}    & 7.0$\pm${\tiny 0.5} & 6.6$\pm${\tiny 1.1}  & 8.4$\pm${\tiny 1.3} & 11.4$\pm${\tiny 1.5}  & 10.5$\pm${\tiny 1.3} & 12.2$\pm${\tiny 1.7} & 11.4$\pm${\tiny 1.8}  \\ 
\midrule
&{\fontfamily{lmtt}\selectfont GenF-2 (Ours)}    & {\bf5.1$\pm${\tiny 0.6}}    & {\bf 6.4$\pm${\tiny 0.5}}  & {\bf 5.8$\pm${\tiny 0.7}}  & {\bf 7.5$\pm${\tiny 0.8}} & 10.2$\pm${\tiny 1.1}  & 9.6$\pm${\tiny 1.3} & 11.7$\pm${\tiny 1.6}  & 10.4$\pm${\tiny 1.9} \\ 

&{\fontfamily{lmtt}\selectfont GenF-5 (Ours)}    & 5.3$\pm${\tiny 0.7}   & 6.6$\pm${\tiny 0.6}  &5.9$\pm${\tiny 0.8}    &  7.7$\pm${\tiny 0.9} & {\bf 9.4$\pm${\tiny 1.0}} & {\bf 9.2$\pm${\tiny 1.1}} & {\bf 10.6$\pm${\tiny 1.3}} & {\bf 9.9$\pm${\tiny 1.4}} \\  

\bottomrule
\end{tabular}
\caption{Performance comparison (MAE, sMAPE (\%) $\pm$ standard deviation over 5 runs) between three variants of GenF and two classical models, five SOTA methods in {\bf predicting greenhouse gas concentrations using the Greenhouse Gas dataset}. Note that all MAE values are multiplied by 100. The bold indicates the best performance.}
\label{performance_4}
\end{table*}

{\bf (3) Strategy Comparison: GenF VS DF/IF.} In Table \ref{performance_mse}, we vary the length of synthetic window for GenF and compare its performance to DF and IF, where DF can be understood as GenF with synthetic window length of zero and IF can be considered as GenF with synthetic window length of $N$ - 1 (i.e., $N$ is the prediction horizon). We use the same canonical transformer to implement both GenF and DF/IF, so as to evaluate the forecasting strategy itself. The forecasting performance is measured using MSE to capture both variance and bias of the predictor. In Table \ref{performance_mse}, we observe that, for the three tasks conducted (i.e., $t + 8$, $t + 12$, $t + 30$), GenF tends to obtain the lowest MSE. As an example, the MSE of GenF-4 at t + 8 is 134, which is 8.8\% and 12.4\% lower than DF and IF, respectively. This verifies our theoretical results in Corollary \ref{corr:1} using experiments, namely that GenF is able to better balance the forecasting variance and bias, leading to a much smaller MSE.  

{\bf (4) Effect of $L$ on Performance.} Adjusting the value of $L$ is a trade-off between forecasting variance and bias. A large value of $L$ brings GenF close to iterative forecasting, while a small value of $L$ brings GenF close to direct forecasting. From Table \ref{performance_mse}, we observe that a smaller $L$ is good for short-range forecasting. As the prediction horizon grows, a larger $L$ tends to be more helpful. For example, in Table \ref{performance_mse}, we find that GenF-4 performs best for $t + 8$ while GenF-6 outperforms other variants for $t + 12$.

\begin{table*}[!h]
\small
\centering
\setlength{\tabcolsep}{9pt}
{\renewcommand{\arraystretch}{0.85}
\begin{tabular}{ll|cc|cc|cc|cc}
\toprule
&{\fontfamily{lmtt}\selectfont Prediction Horizon} & \multicolumn{2}{c|}{{\fontfamily{lmtt}\selectfont t + 8}} & \multicolumn{2}{c|}{{\fontfamily{lmtt}\selectfont t + 12}} & \multicolumn{2}{c|}{{\fontfamily{lmtt}\selectfont t + 15}} & \multicolumn{2}{c}{{\fontfamily{lmtt}\selectfont t + 18}}  \\ \midrule 
&{\fontfamily{lmtt}\selectfont Metrics} & {\fontfamily{lmtt}\selectfont MAE} & {\fontfamily{lmtt}\selectfont sMAPE} &{\fontfamily{lmtt}\selectfont MAE} & {\fontfamily{lmtt}\selectfont sMAPE} &{\fontfamily{lmtt}\selectfont MAE} & {\fontfamily{lmtt}\selectfont sMAPE} &{\fontfamily{lmtt}\selectfont MAE} & {\fontfamily{lmtt}\selectfont sMAPE} \\ 
\cmidrule(lr){1-2} \cmidrule(lr){3-4}\cmidrule(lr){5-6}\cmidrule(lr){7-8}\cmidrule(lr){9-10}

 &{\fontfamily{lmtt}\selectfont ARIMA}      & 5.3$\pm${\tiny 0.7} & 5.9$\pm${\tiny 0.6} & 7.6$\pm${\tiny 0.5} & 8.5$\pm${\tiny 0.7} & 9.2$\pm${\tiny 0.6} & 11.4$\pm${\tiny 1.1} & 10.9$\pm${\tiny 0.8} & 13.5$\pm${\tiny 2.1} \\  
 
 &{\fontfamily{lmtt}\selectfont LSTM} & 4.7$\pm${\tiny 0.5} & 5.4$\pm${\tiny 0.3} & 6.8$\pm${\tiny 0.4} & 7.7$\pm${\tiny 0.8}  & 7.9$\pm${\tiny 0.5} & 10.1$\pm${\tiny 0.9}   & 9.4$\pm${\tiny 0.6} & 12.3$\pm${\tiny 1.9}    \\  \midrule

 \parbox[t]{2mm}{\multirow{5}{*}{\rotatebox[origin=c]{90}{\vspace*{-5mm}{\bf SOTA}}}} & {\fontfamily{lmtt}\selectfont TLSTM}            & 4.2$\pm$ {\tiny 0.2}          & 5.0$\pm$ {\tiny 0.3} & 5.7$\pm$ {\tiny 0.1}          & 7.3$\pm$ {\tiny 0.5} & 6.9$\pm$ {\tiny 0.3}          & 8.9$\pm$ {\tiny 1.1} & 8.2$\pm$ {\tiny 0.3}       & 10.8$\pm$ {\tiny 1.7}  \\ 
&{\fontfamily{lmtt}\selectfont LSTNet}           & 4.0$\pm$ {\tiny 0.1}          & 4.7$\pm$ {\tiny 0.4}  & 5.2$\pm$ {\tiny 0.2}          & 6.2$\pm$ {\tiny 0.8} & 6.8$\pm$ {\tiny 0.3}         & 8.7$\pm$ {\tiny 0.8} & 7.9$\pm$ {\tiny 0.3}       &    10.5$\pm$ {\tiny 1.4} \\ 
&{\fontfamily{lmtt}\selectfont DeepAR}       & 4.3$\pm$ {\tiny 0.2}         &  5.1$\pm$ {\tiny 0.4} & 5.4$\pm$ {\tiny 0.3}          & 6.3$\pm${\tiny 0.5} & 7.1$\pm$ {\tiny 0.4}          & 9.0$\pm$ {\tiny 1.0} & 8.5$\pm$ {\tiny 0.4}       & 10.9$\pm${\tiny 1.8}   
\\ 
&{\fontfamily{lmtt}\selectfont Informer}       & 3.9$\pm$ {\tiny 0.3}         &  4.6$\pm$ {\tiny 0.2} & 4.8$\pm$ {\tiny 0.2}          & 5.9$\pm${\tiny 0.3} & 6.4$\pm$ {\tiny 0.5}          & 8.1$\pm$ {\tiny 0.8} & 7.3$\pm$ {\tiny 0.4}       & 9.4$\pm${\tiny 1.2}   
\\ 
&{\fontfamily{lmtt}\selectfont LogSparse}       & 4.1$\pm$ {\tiny 0.3}         &  4.7$\pm$ {\tiny 0.5} & 4.9$\pm$ {\tiny 0.3}          & 6.1$\pm${\tiny 0.4} & 6.6$\pm$ {\tiny 0.2}          & 8.2$\pm$ {\tiny 0.9} & 7.6$\pm$ {\tiny 0.2}       & 9.8$\pm${\tiny 1.5}   
\\ \midrule

&{\fontfamily{lmtt}\selectfont GenF-2 (Ours)}           & {\bf 3.6$\pm$ {\tiny 0.2}}   & {\bf 4.4$\pm$ {\tiny 0.6}} & 4.6$\pm$ {\tiny 0.2}   & 5.5$\pm$ {\tiny 0.7}  & 6.1$\pm$ {\tiny 0.2}          & 7.8$\pm$ {\tiny 1.2} & 7.1$\pm$ {\tiny 0.5}        &  8.9$\pm$ {\tiny 1.1} \\ 
&{\fontfamily{lmtt}\selectfont GenF-3 (Ours)}          & 3.6$\pm$ {\tiny 0.2}           & 4.5$\pm${\tiny 0.4} & {\bf 4.5$\pm$ {\tiny 0.3}}    & {\bf 5.4$\pm$ {\tiny 0.6}}  & 5.9$\pm$ {\tiny 0.3}  &  7.6$\pm$ {\tiny 1.0} & 6.8$\pm$ {\tiny 0.4} &  8.5$\pm$ {\tiny 0.8} \\ 
&{\fontfamily{lmtt}\selectfont GenF-5 (Ours)}           & 3.7$\pm$ {\tiny 0.2}           & 4.5$\pm$ {\tiny 0.5} & 4.6$\pm$ {\tiny 0.4}          & 5.5$\pm$ {\tiny 0.4}  & {\bf 5.8$\pm$ {\tiny 0.3}}  &  {\bf 7.5$\pm$ {\tiny 0.9}} & {\bf 6.5$\pm$ {\tiny 0.3}} & {\bf 8.4$\pm$ {\tiny 0.7}} \\   
\bottomrule
\end{tabular}}
\caption{Performance (MAE, sMAPE $\pm$ standard deviation over 5 runs) of {\bf predicting fossil fuel consumption using the World Energy Consumption dataset}. Due to the time span of the dataset (from 1971 to 2014), we only show predictive performance up to t + 18. The bold indicates the best performance.}
\label{performance_1}
\end{table*}

\begin{table*}[!ht]
\small
\centering
\setlength{\tabcolsep}{10pt}
\begin{tabular}{ll|cc|cc|cc|cc}
\toprule
&{\fontfamily{lmtt}\selectfont Prediction Horizon} & \multicolumn{2}{c|}{{\fontfamily{lmtt}\selectfont t + 8}} & \multicolumn{2}{c|}{{\fontfamily{lmtt}\selectfont t + 12}} & \multicolumn{2}{c|}{{\fontfamily{lmtt}\selectfont t + 30}} & \multicolumn{2}{c}{{\fontfamily{lmtt}\selectfont t + 60}}  \\ \midrule 
 &{\fontfamily{lmtt}\selectfont Metrics} & {\fontfamily{lmtt}\selectfont MAE} & {\fontfamily{lmtt}\selectfont sMAPE} &{\fontfamily{lmtt}\selectfont MAE} & {\fontfamily{lmtt}\selectfont sMAPE} &{\fontfamily{lmtt}\selectfont MAE} & {\fontfamily{lmtt}\selectfont sMAPE} &{\fontfamily{lmtt}\selectfont MAE} & {\fontfamily{lmtt}\selectfont sMAPE} \\ 
\cmidrule(lr){1-2} \cmidrule(lr){3-4}\cmidrule(lr){5-6}\cmidrule(lr){7-8}\cmidrule(lr){9-10}

 &{\fontfamily{lmtt}\selectfont ARIMA}      & 5.0$\pm${\tiny 0.9} & 6.9$\pm${\tiny 1.0} & 5.8$\pm${\tiny 1.3} & 7.7$\pm${\tiny 1.4} & 9.5$\pm${\tiny 1.6} & 11.5$\pm${\tiny 1.7} & 18.9$\pm${\tiny 2.8} & 25.4$\pm${\tiny 2.9} \\  
 
 &{\fontfamily{lmtt}\selectfont LSTM} & 4.6$\pm${\tiny 0.7} & 6.3$\pm${\tiny 0.9} & 5.2$\pm${\tiny 1.1} & 7.4$\pm${\tiny 1.3}  & 9.0$\pm${\tiny 1.4} & 11.2$\pm${\tiny 1.6}   & 17.6$\pm${\tiny 2.1} & 23.6$\pm${\tiny 2.6}    \\  \midrule

\parbox[t]{2mm}{\multirow{5}{*}{\rotatebox[origin=c]{90}{\vspace*{-5mm}{\bf SOTA}}}}& {\fontfamily{lmtt}\selectfont TLSTM}        & 4.4$\pm${\tiny 0.6}   & 6.1$\pm${\tiny 0.7} & 4.9$\pm${\tiny 0.9}  & 7.1$\pm${\tiny 1.0} & 8.8$\pm${\tiny 1.3}    & 11.0$\pm${\tiny 1.5} & 16.9$\pm${\tiny 2.3} & 23.1$\pm${\tiny 2.3} \\
&{\fontfamily{lmtt}\selectfont LSTNet}    & 4.8$\pm${\tiny 0.9}    & 6.4$\pm${\tiny 1.0} &5.1$\pm${\tiny 0.8}  & 7.2$\pm${\tiny 1.0} & 8.7$\pm${\tiny 1.2}   & 10.9$\pm${\tiny 1.4} & 16.4$\pm${\tiny 2.1} & 22.7$\pm${\tiny 2.0}   \\ 
&{\fontfamily{lmtt}\selectfont DeerAR}     & 4.7$\pm${\tiny 0.8}    & 6.2$\pm${\tiny 0.9} & 4.8$\pm${\tiny 0.8}  & 7.0$\pm${\tiny 0.9} & 8.5$\pm${\tiny 1.0}  & 10.6$\pm${\tiny 1.2} & 16.1$\pm${\tiny 1.9} & 22.1$\pm${\tiny 2.0}  \\ 
&{\fontfamily{lmtt}\selectfont Informer}     & 3.8$\pm${\tiny 0.6}    & 5.9$\pm${\tiny 0.8} & 4.9$\pm${\tiny 1.0}  & 7.1$\pm${\tiny 0.9} & 8.6$\pm${\tiny 1.4}  & 10.8$\pm${\tiny 1.6} & 15.3$\pm${\tiny 2.4} & 19.8$\pm${\tiny 2.7}  \\

&{\fontfamily{lmtt}\selectfont LogSparse}     & 4.5$\pm${\tiny 0.5}    & 6.2$\pm${\tiny 0.6} & 5.2$\pm${\tiny 1.1}  & 7.4$\pm${\tiny 1.2} & 8.8$\pm${\tiny 1.0}  & 10.9$\pm${\tiny 1.1} & 15.7$\pm${\tiny 2.0} & 20.9$\pm${\tiny 1.8}  \\ 
\midrule

&{\fontfamily{lmtt}\selectfont GenF-2 (Ours)}    & {\bf 3.6$\pm${\tiny 0.5}}  & {\bf 5.7$\pm${\tiny 0.6}}  & {\bf 4.4$\pm${\tiny 0.8}}       & {\bf 6.6$\pm${\tiny 1.2}}                  & 8.4$\pm${\tiny 1.2} & 11.0$\pm${\tiny 2.3} & 15.2$\pm${\tiny 2.3} & 19.4$\pm${\tiny 3.0}  \\ 
&{\fontfamily{lmtt}\selectfont GenF-6 (Ours)}    & 3.9$\pm${\tiny 0.8}   & 5.9$\pm${\tiny 0.9}  & 4.6$\pm${\tiny 0.9}    &6.9$\pm${\tiny 1.1} & {\bf 8.1$\pm${\tiny 1.4}} & {\bf 10.6$\pm${\tiny 1.5}} & {\bf 14.1$\pm${\tiny 2.0}} & {\bf 18.5$\pm${\tiny 2.5}} \\  
\bottomrule
\end{tabular}
\caption{Performance (MAE, sMAPE (\%) $\pm$ standard deviation over 5 runs) of {\bf predicting electricity consumption using the House Electricity Consumption dataset}. Note that all MAE values are divided by $10^{11}$ as the readings are cumulative and large. The bold indicates the best performance.}
\vspace{5mm}
\label{performance_3}
\end{table*}

\begin{table*}[!bt]
\small
\centering
\setlength{\tabcolsep}{16pt}
{\renewcommand{\arraystretch}{1.1}
\begin{tabular}{l|cc|ccccc}
\toprule
& \multicolumn{1}{c}{{\fontfamily{lmtt}\selectfont DF}} & \multicolumn{1}{c|}{{\fontfamily{lmtt}\selectfont IF}} & \multicolumn{1}{c}{{\fontfamily{lmtt}\selectfont GenF-2}} & \multicolumn{1}{c}{{\fontfamily{lmtt}\selectfont GenF-4}}  & \multicolumn{1}{c}{{\fontfamily{lmtt}\selectfont GenF-6}} & \multicolumn{1}{c}{{\fontfamily{lmtt}\selectfont GenF-8}} & \multicolumn{1}{c}{{\fontfamily{lmtt}\selectfont GenF-10}} \\ \midrule

\multicolumn{1}{l|}{{\fontfamily{lmtt}\selectfont t + 8}}  & 147 $\pm$ {\tiny 7}  & 153 $\pm$ {\tiny 10}  & 138 $\pm$ {\tiny 9} & 134 $\pm$ {\tiny 7} & 141 $\pm$ {\tiny 9} & --  & --    \\  

\multicolumn{1}{l|}{{\fontfamily{lmtt}\selectfont t + 12}}  & 168 $\pm$ {\tiny 11}  & 177 $\pm$ {\tiny 15} & 159 $\pm$ {\tiny 13} & 156 $\pm$ {\tiny 11} & 152 $\pm$ {\tiny 13} & 157 $\pm$ {\tiny 11} & 166 $\pm$ {\tiny 10}  \\ 
 
\multicolumn{1}{l|}{{\fontfamily{lmtt}\selectfont t + 30}}  & 206 $\pm$ {\tiny 19}  &  221 $\pm$ {\tiny 24} & 189 $\pm$ {\tiny 17} & 185 $\pm$ {\tiny 16} & 179 $\pm$ {\tiny 19} & 188 $\pm$ {\tiny 15} & 186 $\pm$ {\tiny 16} \\ 

\bottomrule
\end{tabular}}
\caption{Performance (MSE $\pm$ standard deviation over 5 runs) comparison between GenF with different lengths of synthetic window and DF/IF for predicting blood pressure on Vital Sign dataset.}
\label{performance_mse}
\end{table*}

{\renewcommand{\arraystretch}{1.0}
\begin{table*}[!t]
\small
%\vspace{7mm}
\centering
\setlength{\tabcolsep}{9pt}
\begin{tabular}{ll|cc|cc|cc|cc}
\toprule
&{\fontfamily{lmtt}\selectfont Prediction Horizon} & \multicolumn{2}{c|}{{\fontfamily{lmtt}\selectfont t + 80}} & \multicolumn{2}{c|}{{\fontfamily{lmtt}\selectfont t + 160}} & \multicolumn{2}{c|}{{\fontfamily{lmtt}\selectfont t + 320}} & \multicolumn{2}{c}{{\fontfamily{lmtt}\selectfont t + 480}}  \\ \midrule 
 &{\fontfamily{lmtt}\selectfont Metrics} & {\fontfamily{lmtt}\selectfont MAE} & {\fontfamily{lmtt}\selectfont sMAPE} &{\fontfamily{lmtt}\selectfont MAE} & {\fontfamily{lmtt}\selectfont sMAPE} &{\fontfamily{lmtt}\selectfont MAE} & {\fontfamily{lmtt}\selectfont sMAPE} &{\fontfamily{lmtt}\selectfont MAE} & {\fontfamily{lmtt}\selectfont sMAPE} \\ 
\cmidrule(lr){1-2} \cmidrule(lr){3-4}\cmidrule(lr){5-6}\cmidrule(lr){7-8}\cmidrule(lr){9-10}

%\parbox[t]{2mm}{\multirow{7}{*}{\rotatebox[origin=c]{90}{\vspace*{-5mm}{\bf SOTA}}}}& {\fontfamily{lmtt}\selectfont TLSTM}        & 5.8$\pm${\tiny 0.8}   & 7.0$\pm${\tiny 0.6} & 6.5$\pm${\tiny 0.9}  & 7.9$\pm${\tiny 1.1} & 11.0$\pm${\tiny 1.2}    & 10.4$\pm${\tiny 1.2} & 12.4$\pm${\tiny 1.8} & 11.8$\pm${\tiny 2.0} \\ 

%&{\fontfamily{lmtt}\selectfont LSTNet}    & 6.0$\pm${\tiny 0.9}    & 6.8$\pm${\tiny 0.5} &6.6$\pm${\tiny 1.0}  & 8.1$\pm${\tiny 1.2} & 11.5$\pm${\tiny 1.6}  & 10.8$\pm${\tiny 1.4} & 12.6$\pm${\tiny 1.5} & 11.7$\pm${\tiny 2.1}   \\

%&{\fontfamily{lmtt}\selectfont DeerAR}     & 6.4$\pm${\tiny 1.0}    & 7.5$\pm${\tiny 0.7} & 6.3$\pm${\tiny 0.9}  & 7.9$\pm${\tiny 1.0} & 11.8$\pm${\tiny 1.7}  & 10.6$\pm${\tiny 1.5} & 11.8$\pm${\tiny 1.6} & 11.2$\pm${\tiny 1.7}  \\ 

&{\fontfamily{lmtt}\selectfont Informer}     & 16.9$\pm${\tiny 2.5}    & 13.5$\pm${\tiny 2.0} & 30.5$\pm${\tiny 9.1}  & 22.1$\pm${\tiny 2.9} & 36.8$\pm${\tiny 10.5}  & 26.1$\pm${\tiny 3.7} & 39.5$\pm${\tiny 14.2} & 30.9$\pm${\tiny 4.5}  \\ 

&{\fontfamily{lmtt}\selectfont LogSparse}     & 17.4$\pm${\tiny 1.9}    & 14.7$\pm${\tiny 1.8} & 32.7$\pm${\tiny 7.7}  & 24.3$\pm${\tiny 3.3} & 35.2$\pm${\tiny 8.9}  & 28.5$\pm${\tiny 3.3} & 38.4$\pm${\tiny 9.8} & 32.1$\pm${\tiny 3.8}  \\ 
\midrule
&{\fontfamily{lmtt}\selectfont GenF-30 (Ours)}    & {\bf15.8$\pm${\tiny 2.1}}    & {\bf 12.4$\pm${\tiny 1.7}}  & {\bf 29.3$\pm${\tiny 8.2}}  & {\bf 21.5$\pm${\tiny 2.8}} & {\bf 33.1$\pm${\tiny 9.9}}  &  {\bf 24.6$\pm${\tiny 2.9}} & {\bf 35.8$\pm${\tiny 12.1}}  & {\bf 26.4$\pm${\tiny 4.2}} \\ 

%&{\fontfamily{lmtt}\selectfont GenF-5 (Ours)}    & 5.3$\pm${\tiny 0.7}   & 6.6$\pm${\tiny 0.6}  &5.9$\pm${\tiny 0.8}    &  7.7$\pm${\tiny 0.9} & {\bf 9.4$\pm${\tiny 1.0}} & {\bf 9.2$\pm${\tiny 1.1}} & {\bf 10.6$\pm${\tiny 1.3}} & {\bf 9.9$\pm${\tiny 1.4}} \\  

\bottomrule
\end{tabular}
\caption{Performance comparison between GenF and two SOTA methods in predicting blood pressure using the MIMIC-III dataset for forecasting horizon up to t + 480.}
\label{new_tab}
\end{table*}}

\begin{figure*}[!t]
\centering
\resizebox{15.5cm}{!}{
\begin{tikzpicture}
\def\ymaxnum{100}
\def\yminnum{40}
\def\auxlineht{98}
\def\arrowheight{95}
\begin{axis}[
    axis line style={line width=0.5pt},
    width=13.8cm,
    height=5.3cm,
    axis lines=left,
    xmin=0,
    xmax=30.5,
    ymin=60,
    ymax=90,
    xtick={0,7,28},
    xticklabels={0,20,28},
    ytick={60,70,80,90},
    yticklabels = {60,70,80,90},
    legend style={nodes={scale=0.9, transform shape}, legend columns=1, at={(1.26,0.1)},anchor=south east,font=\small, legend style={/tikz/every even column/.append style={column sep=0.001cm, row sep = 0.03cm}}},
    legend image post style={line width=0.5mm},
    legend cell align=left,
    xlabel near ticks,
    x label style={at={(1,0)},},
    xlabel={  \hspace{2mm} time step ($t$)},
    ylabel near ticks,
    ylabel={Heart Rate},
]
    \addplot [forget plot]coordinates{(7,0)(7,\auxlineht)};
    \addplot [forget plot]coordinates{(14.87,0)(14.87,\auxlineht)};
    \addplot [forget plot]coordinates{(28,0)(28,\auxlineht)};
     
    \addplot [forget plot,<->]coordinates{(0,89)(7,89)}
       node[coordinate,name=a,pos=0.5]{};
    \addplot [forget plot,<->]coordinates{(14.87,89)(7,89)}
       node[coordinate,name=b,pos=0.5]{};
    \addplot [forget plot,<->]coordinates{(28,89)(14.87,89)}
       node[coordinate,name=c,pos=0.5]{};

    \addplot [cyan,thick, line width = 1pt] table[x=modify, y=True_Value] {table.tex};
    \addplot [red,thick, line width = 1pt,dashed] table[x=modify,y=cWGAN_GEP] {table.tex};  
    \addplot [black,thick, line width = 1pt,dashed] table[x=modify, y=cWGAN_GEP_w/o_MI] {table.tex}; 
    
    \addplot [green,thick, line width = 1pt,dashed] table[x=modify, y=cWGAN_GP] {table.tex}; 
    \addplot [orange,thick, line width = 1pt,dashed] table[x=modify, y=GAN] {table.tex}; 
    
    \addplot [blue,thick, line width = 1pt,dashed] table[x=modify,y=LSTM] {table.tex};
    \addplot [teal,thick, line width = 1pt,dashed] table[x=modify,y=ARIMA] {table.tex};    
    %\addplot [violet,thick, line width = 1pt,dashed] table[x=modify, y=SVR] {table.tex};
    %\addplot [brown,thick, line width = 1pt,dashed] table[x=modify,y=GPR] {table.tex};

   % , SVR (35.6/22.9),GPR (16.9/18.1)
    \legend{{True Value (0/0), {\bf CWGAN-TS (7.1/5.3)}, CWGAN-RS (18.7/8.8), CWGAN-GP (32.8/16.1),  GAN (36.6/20.6), LSTM (20.0/10.8), ARIMA (75.1/27.2)}}
    
    \addplot [forget plot]coordinates{(0,60)}
        node[coordinate,name=A,pos=0.5]{};
    \addplot [forget plot]coordinates{(0,89)}
        node[coordinate,name=B,pos=0.5]{};
    \addplot [forget plot]coordinates{(7,89)}
        node[coordinate,name=C,pos=0.5]{};

    \addplot [forget plot]coordinates{(7,60)}
        node[coordinate,name=D,pos=0.5]{};
    \addplot [forget plot]coordinates{(28,60)}
        node[coordinate,name=E,pos=0.5]{};
    \addplot [forget plot]coordinates{(28,89)}
        node[coordinate,name=F,pos=0.5]{};

    \addplot [forget plot]coordinates{(14.87,89)}
        node[coordinate,name=G,pos=0.5]{};
    \addplot [forget plot]coordinates{(14.87,60)}
        node[coordinate,name=H,pos=0.5]{};
    \addplot [forget plot]coordinates{(28,89)}
        node[coordinate,name=J,pos=0.5]{};
        
\end{axis}
\node at(a)[align=center,above]{\small observation window};
\node at(b)[align=center,above,yshift=-0.8mm]{\small synthetic window};
\node at(c)[align=center,above,yshift=-0.8mm]{\small prediction horizon};

\node at(b)[xshift=-5.5mm,  yshift=-39.1mm]{21};
\draw[line width=0.07mm,  gray](3.74, -0.008) -- (3.74, -0.075);
  
\node at(b)[xshift=4.8mm,  yshift=-39.1mm]{22};
\draw[line width=0.07mm,  gray](4.77, -0.008) -- (4.77, -0.075);

\node at(b)[xshift=15.1mm,  yshift=-39.1mm]{23};
\draw[line width=0.07mm,  gray](5.80, -0.008) -- (5.80, -0.075);

\node at(b)[xshift=25.4mm,  yshift=-39.1mm]{24};
\draw[line width=0.07mm,  gray](6.83, -0.008) -- (6.83, -0.075);

\node at(b)[xshift=35.7mm,  yshift=-39.1mm]{25};
\draw[line width=0.07mm,  gray](7.86, -0.008) -- (7.86, -0.075);

\node at(b)[xshift=46mm,  yshift=-39.1mm]{26};
\draw[line width=0.07mm,  gray](8.89, -0.008) -- (8.89, -0.075);

\node at(b)[xshift=56.3mm,  yshift=-39.1mm]{27};
\draw[line width=0.07mm,  gray](9.92, -0.008) -- (9.92, -0.075);

\begin{scope}[on background layer]
    \fill[top color=gray!90!green!20!white,bottom color=white] (A) rectangle (C);
    \fill[top color=green!80!red!20!white,bottom color=white] (D) rectangle (G);
    \fill[top color=green!70!blue!30!white,bottom color=white] (H) rectangle (J);
\end{scope}
\end{tikzpicture}}
\caption{Synthetic data generated by various models and their corresponding forecasting performance using Vital Sign dataset. The values in parentheses are the MSE of synthetic data generation and forecasting performance, averaged over the synthetic window and prediction horizon, respectively. }%See Fig. \ref{performance_890} in the {\bf Appendix} for results on other subjects.}
\label{performance_all}
\medskip
\hrulefill
\vspace{-6mm}
\end{figure*}
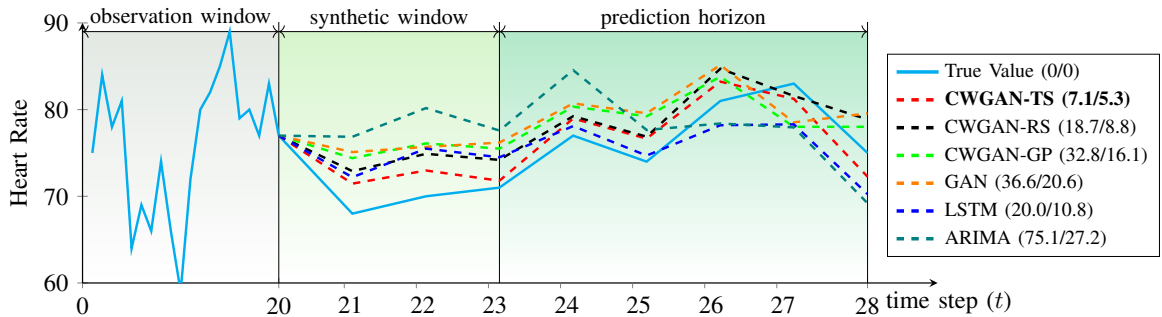

{\bf (5) Long Forecasting Horizons up to $t+480$.} We have discussed our vital sign forecasting results with medical experts and the feedback is that the forecasting results up to $t + 12$ are useful for ICU and high-risk patients. For longer forecasting horizons, the results tend to become unreliable and provide less useful information. That is why we only show the performance up to $t + 60$ in performance evaluation. However, to demonstrate the advantage of GenF over SOTA methods on extremely long forecasting horizons, we present another set of new results in Table \ref{new_tab}.

We shortlist two SOTA methods (i.e., Informer and LogSparse) and compare the performance of GenF to them for longer forecasting horizons (i.e., up to t + 480). The performance of forecasting blood pressure using the MIMIC-III dataset is summarized in Table \ref{new_tab}. We observe that, for long forecasting horizons, GenF still can outperform SOTA methods such as LogSparse and Informer, with an improvement of  6\% - 10\% at t + 320.

\subsection{Ablation Study}
\label{SD}
We now conduct an ablation study to demonstrate the effectiveness of key components in GenF. Specifically, we remove one component at a time in GenF and observe the impact on synthetic time series data generation and forecasting performance. We construct three variants of CWGAN-TS as follows: {\bf (i) CWGAN-GP:} CWGAN-TS without the supervised error penalty term in the loss function \eqref{loss_2}. {\bf (ii) CWGAN-RS:} CWGAN-TS without the ITC algorithm, instead, the CWGAN-TS is trained with a randomly selected training subset. {\bf (iii) GAN:} A conventional GAN \cite{goodfellow2014generative} without considering the Wasserstein distance. More importantly, {\bf the advantage of GAN loss} in generating synthetic time series data is demonstrated by comparing to LSTM.

{\bf (1) Experiment Setup. } In the experiment of predicting heart rate on the Vital Sign dataset, we shortlist a unit called Subject ID 23 and show the observation window ($M$ = 20), synthetic window ($L$ = 3) and the prediction horizon ($N$ = 1, ..., 5) in Fig. \ref{performance_all}.
%evaluate the synthetic data generated by CWGAN-TS and its three variants, as well as the four classical models. Furthermore, we also study the impact of synthetic data on forecasting performance. 
Specifically, in synthetic window, we show the synthetic data generated by CWGAN-TS, its three variants and two classical models (ARIMA and LSTM) from $t = 20$ to $t = 23$. In the prediction horizon, we use the same predictor to evaluate the forecasting performance of all models by taking both of their past observations and the synthetic data (from $t = 21$ to $23$) as the input. As a comparison, we also show the true value (solid line).
In Fig. \ref{performance_all}, the values in parentheses are the MSE of the synthetic data generation and the forecasting performance, averaged over the synthetic window and prediction horizon, respectively. 

{\bf (2) Advantage of CWGAN-TS in Synthetic Data Generation. } 
Fig. \ref{performance_all} shows the synthetic data generation performance of various methods (left number in parentheses). Comparing CWGAN-TS to LSTM, we find that both of them can capture the rising trend of heart rate, but the synthetic data generated by CWGAN-TS is more stable and accurate. Specifically, the performance of LSTM and CWGAN-TS at $t = 21$ (the first synthetic data) are comparable. Subsequently, the synthetic data generated by LSTM tends to fluctuate greatly, resulting in a larger MSE. We posit this is caused by error propagation. As for CWGAN-TS, we see that CWGAN-TS is essentially a combination of LSTMs (see Section \ref{ES}), with the key difference being that CWGAN-TS incorporates an additional GAN loss, but results in more stable synthetic data with a 65\% lower MSE than LSTM. This demonstrates {\bf the effectiveness of the GAN loss in mitigating error propagation}. Interestingly, the LSTM has 3.5K parameters while CWGAN-TS has only 3K parameters. 

When comparing CWGAN-TS to CWGAN-GP, we find that the squared error penalty term in \eqref{loss_2} significantly improves generation performance (i.e., 78\%), suggesting the crucial role of the error penalty term. 
Moreover, when comparing CWGAN-TS to CWGAN-RS, the experimental results show that the ITC algorithm improves generation performance by 62\%, indicating its effectiveness in selecting representative units. 
Comparing to the original GAN, CWGAN-TS improves generation performance by 81\%. We posit this is because CWGAN-TS uses Wasserstein distance as part of the loss function, leading to a more stable learning process.

{\bf (3) Forecasting Performance. } Fig. \ref{performance_all} also shows the forecasting performance (the right number in parentheses). As expected, the model that generates more accurate synthetic data tends to have better forecasting performance, suggesting the important role of the CWGAN-TS and the ITC algorithm in improving long-range forecasting. As an example, CWGAN-TS achieves a much smaller error in synthetic data generation (compared to CWGAN-RS), leading to a 40\% improvement in forecasting performance. 

{\bf (4) Averaged Performance over All Subjects. } We also study the performance averaged over all subjects. In addition to the classical models, the five selected SOTA baselines and a SOTA GAN based model called TimeGAN \cite{yoon2019time} are also examined. We summarize the average performance in Table \ref{ablation_study_new} in the Appendix, where we observe that CWGAN-TS outperforms benchmarks by up to 10\% in synthetic data generation and long-range forecasting.

\begin{table*}[t]
\small
\begin{center}
\setlength{\tabcolsep}{16pt}
{\renewcommand{\arraystretch}{1.0}
\begin{tabular}{l|ccc|cccccc}
\toprule
{\fontfamily{lmtt}\selectfont Performance (MSE) } & \multicolumn{3}{c|}{{\fontfamily{lmtt}\selectfont Generation}} & \multicolumn{4}{c}{{\fontfamily{lmtt}\selectfont Forecasting}} \\ \midrule
& {\fontfamily{lmtt}\selectfont t + 1}  & {\fontfamily{lmtt}\selectfont t + 2}  & {\fontfamily{lmtt}\selectfont t + 3}  & {\fontfamily{lmtt}\selectfont t + 4}  & {\fontfamily{lmtt}\selectfont t + 8}  & {\fontfamily{lmtt}\selectfont t + 12}  & {\fontfamily{lmtt}\selectfont t + 24}  \\
\midrule
{\fontfamily{lmtt}\selectfont ARIMA }  & 106.3 & 125.7 & 145.3 & 160.2 & 193.2 & 201.4 & 225.7 \\
{\fontfamily{lmtt}\selectfont LSTM }                            &  95.2 & 110.6 & 127.8 & 142.8 & 169.6 & 187.3 & 199.4\\ \midrule
{\fontfamily{lmtt}\selectfont LSTNet }                         &  91.3 & 107.4 & 116.8 & 136.8 & 164.3 & 179.4 & 189.3 \\
{\fontfamily{lmtt}\selectfont TLSTM }                              &  90.2 & 106.5 & 113.2 & 129.3 & 159.4 & 169.8 & 185.7\\
{\fontfamily{lmtt}\selectfont DeepAR }                              &  90.9 & 109.4 & 118.4 & 134.3 & 165.8 & 176.8 & 193.6\\
{\fontfamily{lmtt}\selectfont Informer }                              &  89.7 & 105.2 & 109.4 & 133.1 & 155.3 & 170.7 & 185.7\\
{\fontfamily{lmtt}\selectfont LogSparse }   &  {\bf 89.3} & 103.7 & 114.8 & 135.3 & 157.2 & 171.2 & 181.7 \\ \midrule
{\fontfamily{lmtt}\selectfont TimeGAN }   &  94.2 & 114.3 & 121.8 & 139.7 & 167.9 & 180.4 & 195.3 \\ 
\midrule
{\fontfamily{lmtt}\selectfont CWGAN-TS (Ours) }                 &  90.7 & {\bf 97.5} & {\bf 101.9} & {\bf 115.2} & {\bf 145.3} & {\bf 161.2} & {\bf 169.4} \\ \midrule
{\fontfamily{lmtt}\selectfont Min Improvement }  & - & 5.97\% & 6.86\% & 10.9\% & 6.44\% & 5.06\% & 6.77\% \\
\bottomrule
\end{tabular}}
\end{center}
\caption{Ablation Study: Averaged results of all subjects in generating/forecasting heart rate using CWGAN-TS and others.}
\label{ablation_study_new}
\end{table*}

\begin{table*}[!b]
\small
\centering
\setlength{\tabcolsep}{12pt}
{\renewcommand{\arraystretch}{1.4}
\begin{tabular}{l|lc|cc|cc|ccc}
\toprule
\multicolumn{1}{c|}{{\fontfamily{lmtt}\selectfont Prediction Horizon}} & \multicolumn{2}{c|}{{\fontfamily{lmtt}\selectfont t + 8}} & \multicolumn{2}{c|}{{\fontfamily{lmtt}\selectfont t + 12}} & \multicolumn{2}{c|}{{\fontfamily{lmtt}\selectfont t + 30}} & \multicolumn{2}{c}{{\fontfamily{lmtt}\selectfont t + 60}} \\ \midrule
{\fontfamily{lmtt}\selectfont Metrics} &{\fontfamily{lmtt}\selectfont MAE} & {\fontfamily{lmtt}\selectfont sMAPE} &{\fontfamily{lmtt}\selectfont MAE} & {\fontfamily{lmtt}\selectfont sMAPE} &{\fontfamily{lmtt}\selectfont MAE} & {\fontfamily{lmtt}\selectfont sMAPE} &{\fontfamily{lmtt}\selectfont MAE} & {\fontfamily{lmtt}\selectfont sMAPE}\\
%\midrule
\cmidrule(lr){1-2} \cmidrule(lr){2-3}\cmidrule(lr){4-5}\cmidrule(lr){6-7}\cmidrule(lr){8-9}
 %{\fontfamily{lmtt}\selectfont Informer}       & 6.5$\pm${\tiny 0.4} &  5.3$\pm${\tiny 0.5} & 7.9$\pm${\tiny 0.7} & 6.6$\pm${\tiny 0.7} & 11.8$\pm${\tiny 1.5} & 10.1$\pm${\tiny 1.4} & 14.1$\pm${\tiny 2.7} & 11.5$\pm${\tiny 1.6} \\ 
 {\fontfamily{lmtt}\selectfont GenF-6-Informer}       & 6.4$\pm${\tiny 0.2} &  5.3$\pm${\tiny 0.4} & 7.7$\pm${\tiny 0.6} & 6.4$\pm${\tiny 0.6} & 11.4$\pm${\tiny 1.4} & 9.9$\pm${\tiny 1.4} & 13.8$\pm${\tiny 2.6} & 11.4$\pm${\tiny 0.9} \\ 
 %{\fontfamily{lmtt}\selectfont LogSparse}      & 6.6$\pm${\tiny 0.7}   & 5.5$\pm${\tiny 0.3} & 8.1$\pm${\tiny 0.5}   & 6.9$\pm${\tiny 0.6}  & 11.6$\pm${\tiny 1.0}  & 9.7$\pm${\tiny 0.9}  & 14.5$\pm${\tiny 2.8} & 11.3$\pm${\tiny 1.2}   \\ 
 {\fontfamily{lmtt}\selectfont GenF-6-LogSparse}      & 6.6$\pm${\tiny 0.5}   & 5.6$\pm${\tiny 0.4} & 8.0$\pm${\tiny 0.4}   & 6.7$\pm${\tiny 0.5}  & 11.7$\pm${\tiny 1.2}  & 9.7$\pm${\tiny 0.8}  & 14.7$\pm${\tiny 2.9} & 11.5$\pm${\tiny 1.4}   \\ \midrule      
 {\fontfamily{lmtt}\selectfont GenF-6$^{*}$ (Ours)}          & 6.3$\pm${\tiny 0.4}  &   5.2$\pm${\tiny 0.3} & {\bf 7.4$\pm${\tiny 0.7}}  & {\bf 6.2$\pm${\tiny 0.6}}   & {\bf 10.7$\pm${\tiny 1.1}} & {\bf 8.9$\pm${\tiny 0.9}}  & {\bf 12.6$\pm${\tiny 2.5}} &  {\bf 10.2$\pm${\tiny 1.4}}  \\ 
 
\bottomrule
\end{tabular}}
\caption{Performance (MAE, sMAPE$\pm$standard deviation) of predicting blood pressure using the MIMIC-III dataset. All methods use CWGAN-TS to generate synthetic data for next 6 time steps. GenF-6-Informer is to use Informer as the predictor, GenF-6-LogSparse is to use LogSparse as the predictor and GenF-6$^{*}$ is to use the canonical transformer as the predictor. }

\label{performance_add}
\end{table*}

\section{Reflections}
\label{discussion}
%We conclude the paper by presenting some discussions.
%GenF is a competitive long-range forecasting strategy that uses GAN-based synthetic data to strike a balance between iterative and direct forecasting. 

In this paper, we propose a competitive long-range forecasting strategy, called GenF, which is able to better balance the forecasting bias and variance, leading to an improvement of  5\% - 11\% in forecasting performance while having 15\% - 50\% less parameters. 
We now conclude the paper by discussing some relevant points and avenues for future research.

%\noindent {\bf Averaged Performance over All Subjects: } We also study the performance {\em averaged over all subjects}. In addition to the four classical models and the three strong baselines, we also explore a state-of-the-art GAN based model called TimeGAN \cite{yoon2019time} to generate synthetic data. We summarize the average performance of synthetic data generation and long-range forecasting  in Table \ref{ablation_study_new} in the Appendix, where we observe that CWGAN-TS outperforms the best performing benchmark by up to 10\% in synthetic data generation and long-range forecasting. 

\noindent {\bf (1) Selection of the Synthetic Window Length $L$: } The performance of GenF depends on the choice of $L$. Our results suggest that as the prediction horizon grows, increasing $L$ could be helpful (i.e., $L$ = 3 for $t + 8$, $L$ = 6 for $t + 12$ in Table \ref{performance_short}).  Theoretically determining the optimal $L$ clearly deserves deeper thought. Alternatively, $L$ can be thought of as a hyper-parameter and tuned via trial and error. 

\noindent {\bf (2) Flexibility of GenF: }In fact, GenF can be considered as a general framework and is flexible enough to support any model as the synthetic data generator and predictor. We evaluate the forecasting performance using Informer/LogSparse as the predictor and find the performance is not comparable to that of GenF which uses the canonical transformer (see Table \ref{performance_add}). We posit this is due to the simplification of transformer complexity introduced by Informer/LogSparse.

\noindent {\bf (3) The ITC Algorithm: } In the ITC algorithm, we only use first order mutual information as the scoring function. We note that other types of scoring functions, such as joint mutual information, conditional mutual information or pairwise mutual information, could be a better choice for the scoring function. We will explore them in our future research.

\newpage
\clearpage
\balance
\bibliographystyle{plain}
\bibliography{bare_jrnl_new_sample4}

\newpage
\clearpage
\appendix
\nobalance
\setcounter{theorem}{0}
\setcounter{prop}{0}
\setcounter{corollary}{0}
\setlength{\belowdisplayshortskip}{\belowdisplayskip}
\section{Proofs of Theoretical Results}
%For Proposition \ref{prop:1}, we provide minor corrections, which lead to minor changes in the Theorem \ref{thm:1} and Corollary \ref{corr:1}, and are described below. Note that these corrections do not alter the form and implications of the results and their discussion in the main paper. 
We now provide proofs for the theoretical results in the main paper.
\begin{prop}\label{prop:1}
Let $S$ be the sum of bias and variance terms, we have $S_{dir} = $ $B_{dir}(N)$ + $V_{dir}(N)$ for direct forecasting and $S_{iter}$ = $B_{iter}(N)$ + $V_{iter}(N)$ for iterative forecasting. For GenF, let $Y_{M-L} = \{X_{L+1}, \cdots, X_{M}\}$ and $Y_{L} = \{{X}_{M+1}, \cdots, {X}_{M+L}\}$ be the past observations and let $\widetilde{Y}_{L} = \{\widetilde{X}_{M+1}, \cdots, \widetilde{X}_{M+L}\}$ be the generated synthetic data for the next $L$ time steps. Let $\gamma(\theta, N-L) =f(\{Y_{M-L},\widetilde{Y}_{L}\},\theta,N-L)-f(\{Y_{M-L},{Y}_{L}\},\theta,N-L) $. Note that $\gamma(\theta,0)=\widetilde{Y}_{N}-Y_{N}$. We then have, 
\begin{align}
    S_{GenF} = & \underbrace{\mathbb{E}_{\theta \sim \Theta,Y}[\gamma(\theta, N-L)^2]}_\text{\small Iterative Forecasting} \nonumber \\ &+ \underbrace{B_{dir} (N - L) + V_{dir} (N - L) }_\text{\small Direct Forecasting}, \label{eq2}
\vspace*{1mm}\end{align}
where $\mathbb{E}_{\theta \sim \Theta}[\gamma(\theta, 0)^2] = S_{iter}=B_{iter}(N)$ + $V_{iter}(N).$
\end{prop}
\begin{proof}
We consider the decomposition mean-squared error at the horizon $N$, neglecting the noise term $Z(N)$, as follows
\begin{align}
&\mathbb{E}_{Y,\theta}[(\widetilde{X}_{M+N}-X_{M+N})^2] \nonumber \\ = & \mathbb{E}_{Y,\theta}[(f(\{Y_{M-L},\widetilde{Y}_{L}\},\theta,N-L)-X_{M+N})^2]  \nonumber \\
=& \mathbb{E}_Y[(f(\{Y_{M-L},\widetilde{Y}_{L}\},\theta,N-L) \nonumber \\ &-(f(\{Y_{M-L},{Y}_{L}\},\theta,N-L) \nonumber \\ 
& + (f(\{Y_{M-L},{Y}_{L}\},\theta,N-L)-X_{M+N})^2] \nonumber \\ 
=&\mathbb{E}_{Y,\theta}[(f(\{Y_{M-L},\widetilde{Y}_{L}\},\theta,N-L) \nonumber \\ 
& -(f(\{Y_{M-L},{Y}_{L}\},\theta,N-L)] \nonumber \\
& + \mathbb{E}_{Y,\theta}[ (f(\{Y_{M-L},{Y}_{L}\},\theta,N-L)-X_{M+N})^2] \nonumber \\
=& \mathbb{E}_{\theta \sim \Theta,Y}[\gamma(\theta, N-L)^2] \nonumber \\ &+ B_{dir} (N-L) + V_{dir} (N-L). 
\end{align}
Here, the third step follows from the fact that the error terms $f(\{Y_{M-L},\widetilde{Y}_{L}\},\theta,N-L)-f(\{Y_{M-L},{Y}_{L}\},\theta,N-L)$ and $f(\{Y_{M-L},{Y}_{L}\},\theta,N-L)-X_{M+N}$ are independent, as $f(\{Y_{M-L},\widetilde{Y}_{L}\},\theta,N-L)-f(\{Y_{M-L},{Y}_{L}\},\theta,N-L)$ depends primarily on the error of the iterative forecaster, which is independent of the direct forecasting error.
\end{proof}

\begin{theorem}\label{thm:1}
We consider the direct forecasting with parameters $\theta_D$, iterative forecasting with parameters $\theta_I$, and the proposed GenF in Proposition \eqref{prop:1}. Assume that the ground truth realization of the forecasting process can be modelled by some $\theta_D^*$ and $\theta_I^*$, and after training, the estimated parameters follow $\theta_D \sim \mathcal{N}(\theta_D^*,\sigma_D^2)$ and $\theta_I \sim \mathcal{N}(\theta_I^*,\sigma_I^2)$. Assume that the iterative forecasting function is 2nd-order $L_1,L_2$-Lipschitz continuous, and the direct forecasting function is first order Lipschitz continuous. Let us denote quadratic recurrence relations of the form $b_{\alpha}(k+1)= b_{\alpha}(k)\left(L_1+1+b_{\alpha}(k)L_2 \right),$ where $b_{\alpha}(1)=\alpha \sigma_I^2$, for any $\alpha\geq0$. Assume that iterative forecasting has zero variance and direct forecasting has zero bias. Then, for some constants $\beta_0,\beta_1,\beta_2\geq0$, which represent the Lipschitz constants of the direct forecasting function, we have $S_{dir} \leq  U_{dir}$, $S_{iter} \leq U_{iter}$ and $S_{GenF} \leq U_{GenF}$, where $U_{dir} =  (N-1)\beta_1 + \sigma_{D}^2\beta_2$, $U_{iter} = b_{\alpha}(N)^2$, and $U_{GenF} = b_{\alpha}(L)^2(\beta_0) + (N-L-1)\beta_1 + \sigma_{D}^2\beta_2$. The quantities $\alpha$ and $\beta_0,\beta_1,\beta_2$ depend on the iterative and direct forecasting functions respectively.
\end{theorem}
\begin{proof}
Let $X_{k:m}=\{X_k,X_{k+1},...,X_m \}$. Let $f_I$ and $f_D$ represent the iterative and direct forecasting functions respectively, with $\theta_I$ and $\theta_D$ as their respective parameters. As stated in the theorem, we assume that the ground truth realization can be expressed via some configuration of both these forecasters. That is, if 
\vspace{1mm}\begin{equation}
    X_{k+1} = f_{GT}(X_{k+1-m:k+1}) + \epsilon
\end{equation} \vspace{1mm}
represent the ground truth realization that generates the data, and for the direct forecasting function with a horizon $N$, let 
\begin{equation}
    \mathbb{E}_{\epsilon} [X_{k+N}] = f_{GT}(X_{k+1-m:k+1},N)
\end{equation}
represent the mean-squared error minimizing function.Then, as per the assumptions stated in the theorem, there exist $\theta_I^*$ and $\theta_D^*$ such that $f_I(X_{k+1-m:k+1},\theta_I) \approx f_{GT}(X_{k+1-m:k+1})$ and $f_D(X_{k+1-m:k+1},N,\theta_D) \approx f_{GT}(X_{k+1-m:k+1},N)$. Thus, we assume that the direct and iterative forecasters are complex enough to have a configuration close to the ground truth realization. Also note that, as we are only estimating $S_{GenF}$, we do not consider the noise term in our estimation. 

With this, we first estimate the bias of the iterative forecasting part of GenF. 
First, as $\theta_I^*$ represents the underlying realization of the process, we have that 
\vspace{1mm}\begin{equation}
    X_{m+1} = f_I(X_{1:m},\theta_I^*) + \epsilon_1, 
\end{equation}

and so on for all subsequent observations, for some $\epsilon_1\sim \mathcal{N}(0,\sigma_0^2)$. As we assume iterative forecasters with low variance, we intend to compute $\left(\mu_{m+k} - \widehat{\mu}_{m+k}\right)^2$, as the variance term is assumed to be insignificant compared to bias. For the trained $f_I$ with parameters $\theta_I \sim \mathcal{N}(\theta_I^*,\sigma_I^2)$. For the predicted sequence we can write, 
\begin{align}
    &\widehat{X}_{m+1} =  f_I(X_{1:m},\theta_I) + \epsilon_1'\\
    &\widehat{X}_{m+2} = f_I(\{ X_{2:m}, \widehat{X}_{m+1}\},\theta_I) +  \epsilon_2',
\end{align}

and similarly for the subsequent time-steps. Note that as the only change i the trained realization is in $\theta_I$, the error term $\epsilon_2$ will not change in its distribution, i.e., $\epsilon_2 \sim \mathcal{N}(0,\sigma_0^2)$. We consider a first-order Lipschitz continuous expansion of $f_I$ w.r.t $\theta_I$ and a second-order expansion w.r.t the data points $X$, as we will see that $\mu_{m+k} - \widehat{\mu}_{m+k}$ is only affected by the second-order Lipschitz term. For simplicity of notation, we only include the first and second order Taylor expansions of the arguments of $f_I$ in our equations, and write up to the second-order Taylor expansions of all terms. Note that this does not change the 2nd-order Lipschitz continuity constraint. We can then write
\vspace{1mm}\begin{align}
\widehat{X}_{m+2} = & f_I(\{ X_{2:m}, \widehat{X}_{m+1}\},\theta_I) +  \epsilon_2'  \nonumber \\
 = & f_I(\{ X_{2:m}, f_I(X_{1:m},\theta_I) + \epsilon_1' \nonumber \\ & + a_1\mathbb{E}[(\theta_I-\theta^*_I)]\},\theta_I) +  \epsilon_2' \nonumber \\
 = & f_I(\{ X_{2:m}, f_I(X_{1:m},\theta_I) + \epsilon_1'\},\theta_I) \nonumber \\ & + L_1^2\mathbb{E}[(\theta_I-\theta^*_I)] + L_2 L_1^2\mathbb{E}[(\theta_I-\theta^*_I)^2]+  \epsilon_2'
\end{align}

Here, the notation $\mathbb{E}[(\theta_I-\theta^*_I)]$ and $\mathbb{E}[(\theta_I-\theta^*_I)^2]$ represent the sum and the sum of squares of the differences between the corresponding parameters. As $\theta_I \sim \mathcal{N}(\theta_I^*,\sigma_I^2)$, we have that 
\begin{align}
\widehat{\mu}_{m+2} =& \mathbb{E}_{\theta_I,X} [ f_I(\{ X_{2:m}, f_I(X_{1:m},\theta_I) + \epsilon_1'\},\theta_I) \nonumber \\ & + L_1^2\mathbb{E}[(\theta_I-\theta^*_I)] + L_2 L_1^2\mathbb{E}[(\theta_I-\theta^*_I)^2]+  \epsilon_2' ] \nonumber \\
\leq & \mathbb{E}_{\theta_I,X} \left [ f_I(\{ X_{2:m}, f_I(X_{1:m},\theta_I) + \epsilon_1'\},\theta_I) + \epsilon_2' \right] \nonumber \\ & +  L_1^2\mathbb{E}_{\theta_I,X}\left[\mathbb{E}[(\theta_I-\theta^*_I)]\right] \nonumber \\ & + L_2 L_1^2\mathbb{E}_{\theta_I,X}\left[\mathbb{E}[(\theta_I-\theta^*_I)^2]\right] \nonumber \\
= & \mu_{m+2} + 0 + L_2 L_1^2\sigma_I^2. 
\end{align}

Due to the assumption of low-variance iterative forecasters, all subsequent 2nd-order Taylor terms in the expansion of the arguments $\widehat{X}_{m+i}$ will be significantly greater than the first order terms. Thus from $m+2$ and onwards, we can substitute the $\widehat{X}_{m+i}$ argument of $f_I$ with $\widehat{\mu}_{m+i}$, when estimating $\widehat{\mu}_{m+3},\widehat{\mu}_{m+4}..$ and so on. Let $\alpha=L_2 L_1^2$. With these considerations, for $\widehat{\mu}_{m+3}$, we can write 
\begin{align}
\widehat{\mu}_{m+3} =& \mathbb{E}_{\theta_I,X} [ f_I(\{ X_{3:m}, f_I(X_{1:m},\theta_I) + \epsilon_1' \nonumber \\ &+ L_1\mathbb{E}[(\theta_I-\theta^*_I)],\mu_{m+2} + \alpha\mathbb{E}[(\theta_I-\theta^*_I)^2]\},\theta_I) +  \epsilon_2' ] \nonumber \\
\leq &\mathbb{E}_{\theta_I,X} [f_I(\{ X_{3:m}, f_I(X_{1:m},\theta_I) + \epsilon_1' ,\mu_{m+2} \nonumber \\ &+ \alpha\mathbb{E}[(\theta_I-\theta^*_I)^2]\},\theta_I) + L_1^2\mathbb{E}[(\theta_I-\theta^*_I)] \nonumber \\ 
& + L_2 L_1^2\mathbb{E}[(\theta_I-\theta^*_I)^2]+L_1\alpha\mathbb{E}[(\theta_I-\theta^*_I)^2] \nonumber \\ &+ L_2 \left(\alpha\mathbb{E}[(\theta_I-\theta^*_I)^2]\right)^2 + \epsilon_2' \nonumber \\ 
\leq & \mathbb{E}_{\theta_I,X} [f_I(\{ X_{3:m}, f_I(X_{1:m},\theta_I) + \epsilon_1' ,\mu_{m+2} \nonumber \\ &+ \alpha\mathbb{E}[(\theta_I-\theta^*_I)^2]\},\theta_I) + 0 \nonumber \\
& + \alpha\sigma_I^2+L_1\alpha\sigma_I^2 + L_2(\alpha\sigma_I^2)^2  + \epsilon_2'  ] \nonumber \\
\leq & \mu_{m+3}+ \alpha\sigma_I^2(1+L_1)+ L_2(\alpha\sigma_I^2)^2
\end{align}

Using this expansion, we can converge to a generalization of the iterative sequence that generates $\mu_{m+i}$, as follows. Let $b_{\alpha}(i)=\widehat{\mu}_{m+i}-\mu_{m+i}$. Based on our previous expansions, note that we can write: 
\begin{align}
    b_{\alpha}(k+1) =  (1+L_1)b_{\alpha}(k) + L_2\left(b_{\alpha}(k)\right)^2\label{quad_map} 
\end{align}

Here, \eqref{quad_map} represents a quadratic recurrence function and using the same we note that the bias term for $f_I$, $(\widehat{\mu}_{m+L}-\mu_{m+L})^2$, can then be written as $(\widehat{\mu}_{m+L}-\mu_{m+L})^2\leq b_{\alpha}(L)^2$. This accounts for the error for the iterative part of GenF. Next, we first estimate the variance of the direct forecasting function $f_D$. We have that 
\begin{align}
&V_{dir}(N-L) = \mathbb{E}_{X,\theta_D}[(f_D(X_{L:m+L},\theta_D,N-L)  \nonumber \\ &- \mathbb{E}_{\theta_D\sim\mathcal{N}(\theta^*_D,\sigma_D^2)}[f_D(X_{L:m+L},\theta_D,N-L)])^2].
\end{align}

As we assume low-bias $f_D$, the mean estimate $\mathbb{E}_{\theta_D\sim\mathcal{N}(\theta^*_D,\sigma_D^2)}[f_D(X_{L:m+L},\theta_D,N-L)]=f_D(X_{L:m+L},\theta^*_D,N-L)$, yielding
\begin{align}
    &\mathbb{E}_{X,\theta_D\sim\mathcal{N}(\theta^*_D,\sigma_D^2)}[(f_D(X_{L:m+L},\theta_D,N-L)  \nonumber \\ &- f_D(X_{L:m+L},\theta^*_D,N-L))^2].
\end{align}

Using the Lipschitz continuity of $f_D$ w.r.t $\theta_D$ (assuming a Lipschitz constant of $\beta_2$), and the Lipschitz continuity of $V_{dir}(N-L)$ w.r.t $N-L$ itself (Lipschitz constant of $\beta_1$), we can write
\vspace{1mm}\begin{align}
    V_{dir}(N-L) \leq & V_{dir}(1) + (N-L-1)\beta_1  \nonumber \\ 
     = & \mathbb{E}_{X,\theta_D\sim\mathcal{N}(\theta^*_D,\sigma_D^2)}[(f_D(X_{L:m+L},\theta_D,1) \nonumber \\  &- f_D(X_{L:m+L},\theta^*_D,1))^2]+ (N-L-1)\beta_1 \nonumber \\
    \leq & \beta_2\mathbb{E}_{\theta_D\sim\mathcal{N}(\theta^*_D,\sigma_D^2)}[(\theta_D -\theta^*_D)^2] \nonumber \\ &+ (N-L-1)\beta_1 \nonumber \\ = &\beta_2\sigma_D^2 + (N-L-1)\beta_1.
\end{align}

Note that we use this two-step Lipschitz decomposition, instead of a single-step decomposition with $\theta_D$, as without any constraints on the horizon $N-L$, the Lipschitz constant of $f_D(X_{L:m+L},\theta_D,N-L)$ w.r.t $\theta_D$, as $N-L\xrightarrow[]{} \infty$, will be very large and thus, vacuous. Through this two-step decomposition, note that the Lipschitz constant $\beta_2$ will not be as large, as it applies to $f_D(X_{L:m+L},\theta_D,1)$ which has a horizon of $1$. Next, we bound the last term $\mathbb{E}_{\theta \sim \Theta}[\gamma(\theta, N-L)^2]$ in Proposition 1, again via the Lipschitz continuity of $f(\{X_{L:m},\widetilde{X}_{m:m+L}\},\theta,N-L)$ w.r.t $\{X_{L:m},\widetilde{X}_{m:m+L}\}$ (assuming a Lipschitz constant of $\beta'_0$), as follows.
\begin{align}
    & \mathbb{E}_{\theta \sim \Theta}[\gamma(\theta, N-L)^2] \nonumber \\ = & \mathbb{E}_{\theta \sim\mathcal{N}(\theta^*_D,\sigma_D^2) }\left[\left( f(\{X_{L:m},\widetilde{X}_{m:m+L}\},\theta,N-L) \right. \right. \nonumber \\ 
    & -f(\{X_{L:m},X_{m:m+L}\},\theta,N-L))^2] \nonumber \\ 
    \leq & \mathbb{E}_{\theta \sim\mathcal{N}(\theta^*_D,\sigma_D^2) }\left[\left(L\beta'_0(\widetilde{X}_{m+L} - X_{m+L})\right)^2\right] \nonumber \\ 
    \leq & (L\beta'_0)^2\mathbb{E}_{\theta \sim\mathcal{N}(\theta^*_D,\sigma_D^2) }\left[\left(\widetilde{X}_{m+L} - X_{m+L})\right)^2\right] \nonumber \\ =& (L\beta'_0)^2(b_{\alpha}(L))^2.
\end{align}

Denoting $\beta_0=(L\beta'_0)^2$, and combining all error upper bounds from the iterative and direct forecasting functions, we finally have
\begin{align}
    S_{GenF} & \leq \beta_0(b_{\alpha}(L))^2 + \beta_2\sigma_D^2 + (N-L-1)\beta_1 \nonumber \\ &= b_{\alpha}(L)^2(\beta_0) + (N-L-1)\beta_1 + \sigma_{D}^2\beta_2.
\end{align}

Denoting this upper bound via $U_{GenF}$, and similarly for $U_{iter}$ and $U_{dir}$, yields the final result. 
\end{proof}

\begin{corollary}\label{corr:1}
$U_{dir}$, $U_{iter}$ and $U_{GenF}$ are as defined in Theorem \ref{thm:1}.
When  $\beta_0<\min\{\beta_1/b_{\alpha}(1)^2,(b_{\alpha}(N)^2-\sigma_D^2\beta_2)/b_{\alpha}(N-1)^2\}$, we have that $U_{GenF} < U_{iter}$ and $U_{GenF} < U_{dir}$, for some $0<L<N$. Furthermore, when $(N-1)\beta_1 + \sigma_{D}^2\beta_2 \approx  b_{\alpha}(N)^2$, we have $U_{GenF} < U_{iter}$ and $U_{GenF} < U_{dir}$, for \textit{any} choice of $0<L<N$.
\end{corollary}
\begin{proof}
First, with regard to the various terms in $U_{GenF}$, we note that $b_{\alpha}(L)^2(\beta_0)$ is an increasing function of $L$, whereas $ (N-L-1)\beta_1 + \sigma_{D}^2\beta_2$ is a decreasing function of $L$. For simplicity of notation, let us denote $U_{GenF}$ via $U_{GenF}(L)$, for a certain choice of $L$. Thus, for $U_{GenF}(L)$ to be less than the value at its extremes ($U_{dir}$ at $L=0$ and $U_{iter}$ at $L=N$) for some $0<L<N$, we must have that $U_{dir}-U_{GenF}(1)>0$ and $U_{iter} - U_{GenF}(N-1)>0$. This yields
\vspace{1mm}\begin{equation}
    \beta_1 - b_{\alpha}(1)^2(\beta_0)>0, \ and
\end{equation}
\vspace{-2mm}\begin{equation}
    b_{\alpha}(N)^2 - b_{\alpha}(N-1)^2(\beta_0)  - \sigma_D^2\beta_2>0,
\end{equation}

respectively, which summarizes to $\beta_0<\min\{\beta_1/b_{\alpha}(1)^2,(b_{\alpha}(N)^2-\sigma_D^2\beta_2)/b_{\alpha}(N-1)^2\}$. Furthermore, we note that when the upper bounds at $L=0$ and $L=N$ are equal, i.e., $(N-1)\beta_1 + \sigma_{D}^2\beta_2 \approx  b_{\alpha}(N)^2$, then assuming $\beta_0<\min\{\beta_1/b_{\alpha}(1)^2,(b_{\alpha}(N)^2-\sigma_D^2\beta_2)/b_{\alpha}(N-1)^2\}$, we must have that $U_{GenF}(L)<(N-1)\beta_1 + \sigma_{D}^2\beta_2$ and $U_{GenF}(L)<b_{\alpha}(N)^2$ for any $0<L<N$. 

\end{proof}

%\vspace*{40mm}
\begin{IEEEbiography}[{\includegraphics[width=1in,height=1.25in,clip,keepaspectratio]{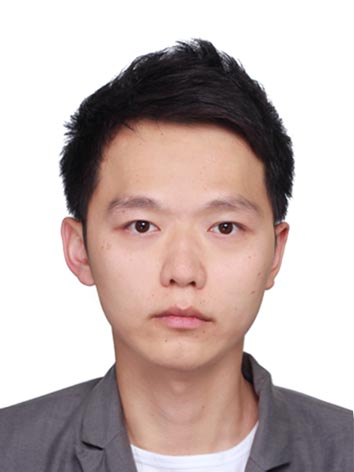}}]%
{Liu Shiyu}

received his B.Eng degree in Electrical Engineering in 2018 from National University of Singapore. Since 2018, he is a PhD candidate at National University of Singapore. His major research interests include feature engineering, time series data forecasting and neural network compression. 
\end{IEEEbiography}

\begin{IEEEbiography}[{\includegraphics[width=1in,height=1.25in,clip,keepaspectratio]{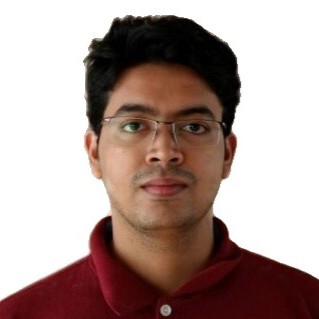}}]%
{Rohan Ghosh} received his PhD degree in Electrical \& Computer Engineering from NUS in 2019, and his B.Tech and M.Tech degrees from Indian Institute of Technology Kharagpur. He is currently a research fellow at NUS working on theoretical problems in machine learning. His research interests include machine learning, computer vision and information theory.  
\end{IEEEbiography}

\begin{IEEEbiography}[{\includegraphics[width=1in,height=1.6in,clip,keepaspectratio]{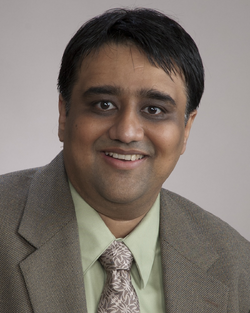}}]%
{Mehul Motani}
received the B.E. degree from Cooper Union, New York, NY, the M.S. degree from Syracuse University, Syracuse, NY, and the Ph.D. degree from Cornell University, Ithaca, NY, all in Electrical and Computer Engineering. Dr. Motani is currently an Associate Professor in the Electrical and Computer Engineering Department at the National University of Singapore (NUS) and a Visiting Research Collaborator at Princeton University.His research interests include information and coding theory, machine learning, biomedical informatics, wireless and sensor networks, and the Internet-of-Things.
\end{IEEEbiography}

\end{document}